\documentclass[runningheads]{llncs}
\usepackage[T1]{fontenc}
\usepackage{lmodern}
\usepackage{graphicx}
\usepackage{hyperref}

\usepackage{color}

\usepackage[utf8]{inputenc}
\usepackage{microtype}
\usepackage{graphicx}
\usepackage{subfigure}
\usepackage{booktabs}

\usepackage{amsmath}
\usepackage{amssymb}
\usepackage{mathtools}
\usepackage{url}
\usepackage{multicol}
\usepackage{multirow}
\usepackage{algorithm}
\usepackage{algorithmic}
\usepackage{dsfont}
\usepackage{paralist}

\usepackage[utf8]{inputenc}
\usepackage[T1]{fontenc}
\usepackage{placeins}

\newcommand{\Stable}{\mathcal{X}_s}
\newcommand{\loss}{\mathcal{L}}
\newcommand{\support}{\mathsf{support}}
\renewcommand{\tilde}{\widetilde}
\newcommand{\eps}{\epsilon}

\usepackage[capitalize,noabbrev]{cleveref}

\begin{document}
\title{Learning Provably Stabilizing Neural Controllers for Discrete-Time Stochastic Systems\thanks{This work was supported in part by the ERC-2020-AdG 101020093, ERC CoG 863818 (FoRM-SMArt) and the European Union’s Horizon 2020 research and innovation programme under the Marie Skłodowska-Curie Grant Agreement No.~665385.}}
\titlerunning{Learning Provably Stabilizing Neural Controllers}
%
\author{Matin Ansaripour\inst{1} \and
Krishnendu Chatterjee\inst{2} \and
Thomas A. Henzinger\inst{2} \and \\
Mathias Lechner\inst{3} \and
\DJ{}or\dj{}e \v{Z}ikeli\'c \inst{2}}
\authorrunning{M.~Ansaripour, K.~Chatterjee, T.~A.~Henzinger, M.~Lechner, \DJ{}.~\v{Z}ikeli\'c}
%
\institute{Sharif University of Technology, Tehran, Iran \\
\email{matinansaripour@gmail.com} \and
Institute of Science and Technology Austria (ISTA), Klosterneuburg, Austria \\ \email{\{krishnendu.chatterjee, tah, djordje.zikelic\}@ist.ac.at} \and
Massachusetts Institute of Technology, Cambridge, MA, USA \\ \email{mlechner@mit.edu}}
\maketitle              

\begin{abstract}
We consider the problem of learning control policies in discrete-time stochastic systems which guarantee that the system stabilizes within some specified stabilization region with probability~$1$. Our approach is based on the novel notion of stabilizing ranking supermartingales (sRSMs) that we introduce in this work. Our sRSMs overcome the limitation of methods proposed in previous works whose applicability is restricted to systems in which the stabilizing region cannot be left once entered under any control policy. We present a learning procedure that learns a control policy together with an sRSM that formally certifies probability~$1$ stability, both learned as neural networks. We show that this procedure can also be adapted to formally verifying that, under a given Lipschitz continuous control policy, the stochastic system stabilizes within some stabilizing region with probability~$1$. Our experimental evaluation shows that our learning procedure can successfully learn provably stabilizing policies in practice.

\end{abstract}
\section{Introduction}
\label{sec:intro}

Machine learning based methods and in particular reinforcement learning (RL) present a promising approach to solving highly non-linear control problems. This has sparked interest in the deployment of learning-based control methods in safety-critical autonomous systems such as self-driving cars or healthcare devices. However, the key challenge for their deployment in real-world scenarios is that they do not consider hard safety constraints. For instance, the main objective of RL is to maximize expected reward~\cite{sutton2018reinforcement}, but doing so provides no guarantees of the system's safety. A more recent paradigm in safe RL considers constrained Markov decision processes (cMDPs)~\cite{altman1999constrained,Geibel06,uchibe2007constrained,achiam2017constrained,ChowNDG18}, which are equiped with both a reward function and an auxiliary cost function. The goal of these works is then to maximize expected reward while keeping expected cost below some tolerable threshold. While these methods do enhance safety, they only ensure empirically that the expected cost function is below the threshold and do not provide any formal guarantees on constraint satisfaction.

This is particularly concerning for safety-critical applications, in which unsafe behavior of the system might have fatal consequences. Thus, a fundamental challenge for deploying learning-based methods in safety-critical autonomous systems applications is {\em formally certifying} safety of learned control policies~\cite{AmodeiOSCSM16,GarciaF15}.

Stability is a fundamental safety constraint in control theory, which requires the system to converge to and eventually stay within some specified stabilizing region with probability~$1$, a.k.a.~almost-sure (a.s.) asymptotic stability~\cite{khalil2002nonlinear,kushner1965stability}. 
Most existing research on learning policies for a control system with formal guarantees on stability considers {\em deterministic} systems and employs Lyapunov functions~\cite{khalil2002nonlinear} for certifying the system's stability.
In particular, a Lyapunov function is learned jointly with the control policy~\cite{BerkenkampTS017,RichardsB018,ChangRG19,AbateAGP21}. 
Informally, a Lyapunov function is a function that maps system states to nonnegative real numbers whose value decreases after every one-step evolution of the system until the stabilizing region is reached. 
Recently,~\cite{lechner2021stability} proposed a learning procedure for learning {\em ranking supermartingales (RSMs)}~\cite{ChakarovS13} for certifying a.s.~asymptotic stability in discrete-time stochastic systems. RSMs generalize Lyapunov functions to supermartingale processes in probability theory~\cite{Williams91} and decrease in value in {\em expectation} upon every one-step evolution of the system. 

While these works present significant advances in learning control policies with formal stability guarantees as well as formal stability verification, they are either only applicable to deterministic systems or assume that the stabilizing set is {\em closed under system dynamics}, i.e., the agent cannot leave it once entered. 
In particular, the work of \cite{lechner2021stability} reduces stability in stochastic systems to an {\em a.s.~reachability} condition by assuming that the agent cannot leave the stabilization set. However, this assumption may not hold in real-world settings because the agent may be able to leave the stabilizing set with some positive probability due to the existence of stochastic disturbances, see Figure~\ref{fig:srsmexample}. We illustrate the importance of relaxing this assumption on the classical example of balancing a pendulum in the upright position, which we also study in our experimental evaluation. The closedness under system dynamics assumption implies that, once the pendulum is in an upright position, it is ensured to stay upright and not move away. However, this is not a very realistic assumption due to possible existence of minor disturbances which the controller needs to balance out. The closedness under system dynamics assumption essentially assumes the existence of a balancing control policy which takes care of this problem. In contrast, our method does not assume such a balancing policy and learns a control policy which ensures that both (1)~the pendulum reaches the upright position and (2)~that the pendulum eventually stays upright with probability~1.

While the removal of the assumption that a stabilizing region cannot be left may appear to be a small improvement, in formal methods this is well-understood to be a significant and difficult step.
With the assumption, the desired controller has an a.s.~reachability objective.
Without the assumption, the desired controller has an a.s.~persistence (or co-B\"uchi) objective, namely, to reach and stay in the stabilizing region with probability~$1$.
Verification or synthesis for reachability conditions allow in general much simpler techniques than verification or synthesis for persistence conditions.
For example, in non-stochastic systems, reachability can be expressed in alternation-free $\mu$-calculus (i.e., fixpoint computation), whereas persistence requires alternation (i.e., nested fixpoint computation).
Technically, reachability conditions are found on the first level of the Borel hierarchy, while persistence conditions are found on the second level~\cite{ChangMP92}.
It is, therefore, not surprising that also over continuous and stochastic state spaces, reachability techniques are insufficient for solving persistence problems.

In this work, we present the following three contributions. 
\begin{compactenum}
\item \textbf{Theoretical contributions} 
In this work, we introduce {\em stabilizing ranking supermartingales (sRSMs)} and prove that they certify a.s.~asymptotic stability in discrete-time stochastic systems even when the stabilizing set is not assumed to be closed under system dynamics. The key novelty of our sRSMs compared to RSMs is that they also impose an expected decrease condition within a part of the stabilizing region. 
The additional condition ensures that, once entered, the agent leaves the stabilizing region with probability at most $p<1$. Thus, we show that the probability of the agent entering and leaving the stabilizing region $N$ times is at most $p^N$, which by letting $N\rightarrow\infty$ implies that the agent eventually stabilizes within the region with probability $1$. The key conceptual novelty is that we combine the convergence results of RSMs which were also exploited in~\cite{lechner2021stability} with a {\em concentration bound} on the supremum value of a supermartingale process. This combined reasoning allows us to formally guarantee a.s.~asymptotic stability even for systems in which the stabilizing region is not closed under system dynamics. We remark that our proof that sRSMs certify a.s.~asymptotic stability is not an immediate application of results from martingale theory, but that it introduces a novel method to reason about eventual stabilization within a set. We present this novel method in the proof of Theorem~\ref{thm:soundness}. 
Finally, we show that sRSMs not only present qualitative results to certify a.s. asymptotic stability but also present quantitative upper bounds on the number of time steps that the system may spend outside of the stabilization set prior to stabilization.

\item \textbf{Algorithmic contributions} 
Following our theoretical results on sRSMs, we present an algorithm for learning a control policy jointly with an sRSM that certifies a.s.~asymptotic stability. The method parametrizes both the policy and the sRSM as neural networks and draws insight from established procedures for learning neural network Lyapunov functions \cite{ChangRG19} and RSMs \cite{lechner2021stability}. 
It loops between a learner module that jointly trains a policy and an sRSM candidate and a verifier module that certifies the learned sRSM candidate by formally checking whether all sRSM conditions are satisfied.
If the sRSM candidate violates some sRSM conditions, the verifier module produces counterexamples that are added to the learner module's training set to guide the learner in the next loop iteration. Otherwise, if the verification is successful and the algorithm outputs a policy, then the policy guarantees a.s.~asymptotic stability. By fixing the control policy and only learning and verifying the sRSM, our algorithm can also be used to verify that a given control policy guarantees a.s.~asymptotic stability. This verification procedure only requires that the control policy is a Lipschitz continuous function.

\item \textbf{Experimental contributions} 
We experimentally evaluate our learning procedure on $2$ stochastic RL tasks in which the stabilizing region is not closed under system dynamics and show that our learning procedure successfully learns control policies with a.s.~asymptotic stability guarantees for both tasks.
\end{compactenum}

\paragraph{Organization} The rest of this work is organized as follows. Section~\ref{sec:prelims} contains preliminaries. In Section~\ref{sec:theory}, we introduce our novel notion of stabilizing ranking supermartingales and prove that they provide a sound certificate for a.s.~asymptotic stability, which is the main theoretical contribution of our work. In Section~\ref{sec:algo}, we present the learner-verifier procedure for jointly learning a control policy together with an sRSM that formally certifies a.s.~asymptotic stability. In Section~\ref{sec:experiments}, we experimentally evaluate our approach. We survey related work in Section~\ref{sec:relatedwork}. Finally, we conclude in Section~\ref{sec:limitconclusion}.


\section{Preliminaries}
\label{sec:prelims}

	We consider a discrete-time stochastic dynamical system of the form
    \[ \mathbf{x}_{t+1} = f(\mathbf{x}_t, \pi(\mathbf{x}_t), \omega_t), \]
    where $f:\mathcal{X}\times\mathcal{U}\times\mathcal{N}\rightarrow\mathcal{X}$ is a dynamics function, $\pi:\mathcal{X}\rightarrow\mathcal{U}$ is a control policy and $\omega_t \in \mathcal{N}$ is a stochastic disturbance vector. Here, we use $\mathcal{X}\subseteq\mathbb{R}^n$ to denote the state space, $\mathcal{U}\subseteq\mathbb{R}^m$ the action space and $\mathcal{N}\subseteq\mathbb{R}^p$ the stochastic disturbance space  of the system. In each time step, $\omega_t$ is sampled according to a probability distribution $d$ over $\mathcal{N}$, independently from the previous samples.
    
    A sequence $(\mathbf{x}_t,\mathbf{u}_t,\omega_t)_{t\in\mathbb{N}_0}$  of state-action-disturbance triples is a trajectory of the system, if $\mathbf{u}_t=\pi(\mathbf{x}_t)$, $\omega_t\in\support(d)$ and $\mathbf{x}_{t+1}=f(\mathbf{x}_t,\mathbf{u}_t,\omega_t)$ hold for each $t\in\mathbb{N}_0$. For each state $\mathbf{x}_0\in\mathcal{X}$, the system induces a Markov process and defines a probability space over the set of all trajectories that start in $\mathbf{x}_0$~\cite{Puterman94}, with the probability measure and the expectation operators $\mathbb{P}_{\mathbf{x}_0}$ and~$\mathbb{E}_{\mathbf{x}_0}$.
    
    \paragraph{Assumptions} The state space $\mathcal{X}\subseteq\mathbb{R}^n$, the action space $\mathcal{U}\subseteq\mathbb{R}^m$ and the stochastic disturbance space $\mathcal{N}\subseteq\mathbb{R}^p$ are all assumed to be Borel-measurable. Furthermore, we assume that the system has a {\em bounded maximal step size} under any policy $\pi$, i.e.~that there exists $\Delta > 0$ such that for every $\mathbf{x}\in\mathcal{X}$, $\omega\in \mathcal{N}$ and policy $\pi$ we have $||\mathbf{x} - f(\mathbf{x}, \pi(\mathbf{x}), \omega)||_1 \leq \Delta$. Note that this is a realistic assumption that is satisfied in many real-world scenarios, e.g.~a self-driving car can only traverse a certain maximal distance within each time step whose bounds depend on the maximal speed that the car can develop.
    
    For our learning procedure in Section~\ref{sec:algo}, we assume that $\mathcal{X}\subseteq\mathbb{R}^n$ is compact and that $f$ is Lipschitz continuous, which are common assumptions in control theory. Given two metric spaces $(X, d_X)$ and $(Y, d_Y)$, a function $g: X \rightarrow Y$ is said to be {\em Lipschitz continuous} if there exists a constant $L>0$ such that for every $x_1,x_2 \in X$ we have $d_Y(g(x_1), g(x_2)) \leq L \cdot d_X(x_1, x_2)$. We say that $L$ is a Lipschitz constant of $g$. For the verification procedure when the control policy $\pi$ is given, we also assume that $\pi$ is Lipschitz continuous. This is also a common assumption in control theory and RL that allows for a rich class of policies including neural network policies, as all standard activation functions such as ReLU, sigmoid or tanh are Lipschitz continuous~\cite{SzegedyZSBEGF13}. Finally, in Section~\ref{sec:algo} we assume that the stochastic disturbance space $\mathcal{N}$ is bounded or that $d$ is a product of independent univariate distributions, which is needed for efficient sampling and expected value computation.
    
    \paragraph{Almost-sure asymptotic stability} There are several notions of stability in stochastic systems. In this work, we consider the notion of almost-sure asymptotic stability~\cite{kushner1965stability}, which requires the system to eventually {\em converge and stay within} the stabilizing set. In order to define this formally, for each $\mathbf{x}\in\mathcal{X}$ let $d(\mathbf{x},\Stable)=\inf_{\mathbf{x_s}\in\Stable}||\mathbf{x}-\mathbf{x}_s||_1$, where $||\cdot||_1$ is the $l_1$-norm on $\mathbb{R}^m$.
    
    \begin{definition}
    A Borel-measurable set $\Stable\subseteq\mathcal{X}$ is {\em almost-surely (a.s.) asymptotically stable}, if for each initial state $\mathbf{x}_0\in\mathcal{X}$ we have
    \[ \mathbb{P}_{\mathbf{x}_0}\Big[ \lim_{t\rightarrow\infty}d(\mathbf{x}_t,\Stable) = 0 \Big] = 1. \]
    \end{definition}
    
    The above definition slightly differs from that of~\cite{kushner1965stability} which considers the special case of a singleton $\Stable=\{\mathbf{0}\}$. The reason for this difference is that, analogously to~\cite{lechner2021stability} and to the existing works on learning stabilizing policies in deterministic systems~\cite{BerkenkampTS017,RichardsB018,ChangRG19}, we need to consider stability with respect to an open neighborhood of the origin for learning to be numerically stable.


\section{Theoretical Results}\label{sec:theory}

We now introduce our novel notion of stabilizing ranking supermartingales (sRSMs). We then show that sRSMs can be used to formally certify a.s.~asymptotic stability with respect to a fixed policy $\pi$ {\em without} requiring that the stabilizing set is closed under system dynamics. To that end, in this section we assume that the policy $\pi$ is fixed.
In the next section, we will then present our learning procedure.

\paragraph{Prior work -- ranking supermartingales (RSMs)} In order to motivate our sRSMs and to explain their novelty, we first recall ranking supermartingales (RSMs)~\cite{ChakarovS13} that were used in~\cite{lechner2021stability} for certifying a.s.~asymptotic stability under a given policy $\pi$, when the stabilizing set is assumed to be closed under system dynamics. If the stabilizing set is assumed to be closed under system dynamics, then a.s.~asymptotic stability of $\Stable$ is equivalent to {\em a.s.~reachability} since the agent cannot leave $\Stable$ once entered. 

Intuitively, an RSM is a non-negative continuous function $V:\mathcal{X} \rightarrow \mathbb{R}$ whose value at each state in $\mathcal{X}\backslash \Stable$ strictly decreases in expected value by some $\epsilon>0$ upon every one-step evolution of the system under the policy $\pi$.

\begin{definition}[Ranking supermartingales~\cite{ChakarovS13,lechner2021stability}]\label{def:rsm}
A continuous function $V: \mathcal{X} \rightarrow \mathbb{R}$ is a {\em ranking supermartingale (RSM) for $\Stable$} if $V(\mathbf{x})\geq 0$ for each $\mathbf{x}\in\mathcal{X}$ and if there exists $\epsilon > 0$ such that for each $\mathbf{x}\in\mathcal{X}\backslash \Stable$ we have
\[ \mathbb{E}_{\omega\sim d}\Big[ V ( f(\mathbf{x}, \pi(\mathbf{x}), \omega) ) \Big] \leq V(\mathbf{x}) - \epsilon.\]
\end{definition}

It was shown that, if a system under policy $\pi$ admits an RSM {\em and} the stabilizing set $\Stable$ is assumed to be closed under system dynamics, then $\Stable$ is a.s.~asymptotically stable. The intuition behind this result is that $V$ needs to strictly decrease in expected value until $\Stable$ is reached while remaining bounded from below by $0$. Results from martingale theory can then be used to prove that the agent must eventually converge and reach $\Stable$ with probability $1$, due to a strict decrease in expected value by $\epsilon>0$ outside of $\Stable$ which prevents convergence to any other state. However, apart from nonnegativity, the defining conditions on RSMs do not impose any conditions on the RSM once the agent reaches $\Stable$. In particular, if the stabilizing set $\Stable$ is {\em not} closed under system dynamics, then the defining conditions of RSMs do not prevent the agent from leaving and reentering $\Stable$ infinitely many times and thus never stabilizing. In order to formally ensure stability, the defining conditions of RSMs need to be strengthened and in the rest of this section we solve this problem.

\paragraph{Our new certificate -- stabilizing ranking supermartingales (sRSMs)} We now define our sRSMs, which certify a.s.~asymptotic stability even when the stabilizing set is not assumed to be closed under system dynamics and thus overcome the limitation of RSMs of~\cite{lechner2021stability} that was discussed above. Recall, we use $\Delta$ to denote the maximal step size of the system.


\begin{definition}[Stabilizing ranking supermartingales]\label{def:stabilizingrsm}
Let $\epsilon, M, \delta > 0$. A Lipschitz continuous function $V: \mathcal{X} \rightarrow \mathbb{R}$ is said to be an {\em $(\epsilon,M,\delta)$-stabilizing ranking supermartingale ($(\epsilon,M,\delta)$-sRSM) for $\Stable$} if the following conditions hold:
\begin{compactenum}
    \item {\em Nonnegativity.} $V(\mathbf{x})\geq 0$ holds for each $\mathbf{x}\in\mathcal{X}$.
    \item {\em Strict expected decrease if $V\geq M$.} For each $\mathbf{x}\in\mathcal{X}$, if $V(\mathbf{x}) \geq M$ then
    \begin{equation*}
    \mathbb{E}_{\omega\sim d}\Big[ V \Big( f(\mathbf{x}, \pi(\mathbf{x}), \omega) \Big) \Big] \leq V(\mathbf{x}) - \epsilon.
    \end{equation*}
    \item {\em Lower bound outside $\Stable$.} $V(\mathbf{x})\geq M + L_V\cdot\Delta + \delta$ holds for each $\mathbf{x}\in\mathcal{X}\backslash\Stable$, where $L_V$ is a Lipschitz constant of $V$.
\end{compactenum}
\end{definition}

\begin{figure*}
    \centering
    \includegraphics[width=\textwidth]{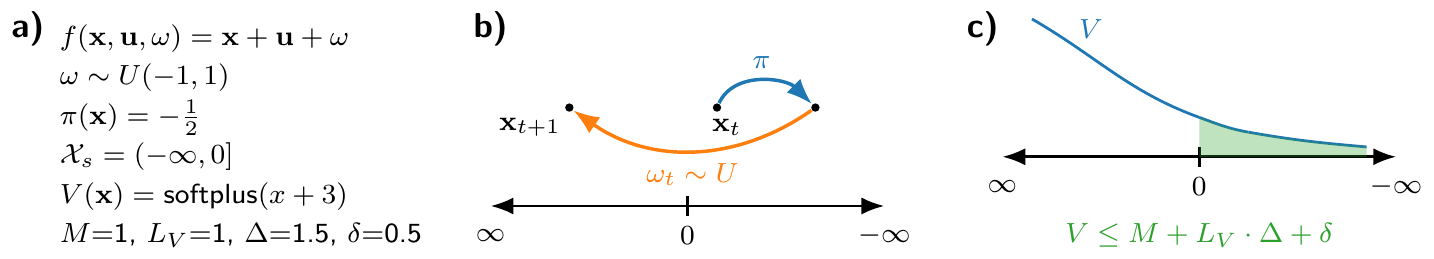}
    \caption{Example of a $1$-dimensional stochastic dynamical system for which the stabilizing set $\Stable$ is not closed under system dynamics since from every system state any other state is reachable with positive probability. \textbf{a)} System definition and an sRSM that it admits. \textbf{b)} Illustration of a single time step evolution of the system. \textbf{c)} Visualization of the sRSM and the corresponding level set used to bound the probability of leaving the stabilizing region. }
    \label{fig:srsmexample}
\end{figure*}

An example of an sRSM for a $1$-dimensional stochastic dynamical system is shown in Figure~\ref{fig:srsmexample}. The intuition behind our new conditions is as follows. Condition~$2$ in Definition~\ref{def:stabilizingrsm} requires that, at each state in which $V\geq M$, the value of $V$ decreases in expectation by $\epsilon>0$ upon one-step evolution of the system. As we show below, this ensures probability~$1$ convergence to the set of states $S = \{\mathbf{x}\in\mathcal{X} \mid V(\mathbf{x})\leq M\}$ from any other state of the system. On the other hand, condition~$3$ in Definition~\ref{def:stabilizingrsm} requires that $V\geq M + L_V\cdot\Delta + \delta$ outside of the stabilizing set $\Stable$, thus $S\subseteq\Stable$. Moreover, if the agent is in a state where $V\leq M$, the value of $V$ in the next state has to be $\leq M + L_V\cdot\Delta$ due to Lipschitz continuity of $V$ and $\Delta$ being the maximal step size of the system. Therefore, even if the agent leaves $S$, for the agent to actually leave $\Stable$ the value of $V$ has to {\em increase} from a value $\leq M + L_V\cdot\Delta$ to a value $\geq M + L_V\cdot\Delta + \delta$ while satisfying the strict expected {\em decrease} condition imposed by condition $2$ in Definition~\ref{def:stabilizingrsm} at every intermediate state that is not contained in $S$. The following theorem is the main result of this section. 
\begin{theorem}\label{thm:soundness}
If there exist $\epsilon, M, \delta > 0$ and an $(\epsilon,M,\delta)$-sRSM for $\Stable$, then $\Stable$ is a.s.~asymptotically stable.
\end{theorem}

\begin{proof}[Proof sketch, full proof in Appendix~\ref{sec:thmproofs}]
    In order to prove Theorem~\ref{thm:soundness}, we need to show that $\mathbb{P}_{\mathbf{x}_0}[ \lim_{t\rightarrow\infty}d(\mathbf{x}_t,\Stable) = 0 ] = 1$ for every $\mathbf{x}_0\in\mathcal{X}$. We show this by proving the following two claims. First, we show that, from each initial state $\mathbf{x}_0\in\mathcal{X}$, the agent converges to and reaches $S = \{\mathbf{x}\in\mathcal{X} \mid V(\mathbf{x})\leq M\}$ with probability~$1$. The set $S$ is a subset of $\Stable$ by condition~$3$ in Definition~3 of sRSMs. Second, we show that once the agent is in $S$ it may leave $\Stable$ with probability at most $p=\frac{M+L_V\cdot \Delta}{M+L_V\cdot \Delta+\delta}<1$. We then prove that the two claims imply Theorem~\ref{thm:soundness}.

    \medskip\noindent{\em Claim 1.} For each intial state $\mathbf{x}_0\in\mathcal{X}$, the agent converges to and reaches $S = \{\mathbf{x}\in\mathcal{X} \mid V(\mathbf{x})\leq M\}$ with probability~$1$.
    
    \medskip\noindent To prove Claim~1, let $\mathbf{x}_0\in\mathcal{X}$. If $\mathbf{x}_0\in S$, then the claim trivially holds. So suppose w.l.o.g.~that $\mathbf{x}_0\not\in S$. We consider the probability space $(\Omega_{\mathbf{x}_0},\mathcal{F}_{\mathbf{x}_0},\mathbb{P}_{\mathbf{x}_0})$ of all system trajectories that start in $\mathbf{x}_0$, and define a {\em stopping time} $T_S: \Omega_{\mathbf{x}_0} \rightarrow \mathbb{N}_0\cup\{\infty\}$ which to each trajectory assigns the first hitting time of the set $S$ and is equal to $\infty$ if the trajectory does not reach $S$. Furthermore, for each $i\in\mathbb{N}_0$, we define a random variable $X_i$ in this probability space via
    \begin{equation}\label{eq:processmain}
        X_i(\rho) = \begin{cases}
        V(\mathbf{x}_i), &\text{if } i<T_S(\rho) \\
        V(\mathbf{x}_{T_S(\rho)}), &\text{otherwise}
        \end{cases}
    \end{equation}
    for each trajectory $\rho=(\mathbf{x}_t,\mathbf{u}_t,\omega_t)_{t\in\mathbb{N}_0}\in\Omega_{\mathbf{x}_0}$. In words, $X_i$ is equal to the value of $V$ at the $i$-th state along the trajectory until $S$ is reached, upon which it becomes constant and equal to the value of $V$ upon first entry into $S$. We prove that $(X_i)_{i=0}^\infty$ is an instance of the mathematical notion of {\em $\epsilon$-ranking supermartingales ($\eps$-RSMs)}~\cite{ChakarovS13} for the stopping time $T_S$. Intuitively, an $\eps$-RSM for $T_S$ is a stochastic process which is non-negative, decreases in expected value upon every one-step evolution of the system and furthermore the decrease is strict and by $\epsilon>0$ until the stopping time $T_S$ is exceeded. If $\epsilon$ is allowed to be $0$ as well, then the process is simply said to be a {\em supermartingale}~\cite{Williams91}. It is a known result in martingale theory that, if an $\epsilon$-RSM exists for $T_S$, then $\mathbb{P}_{\mathbf{x}_0}[T_S<\infty]=\mathbb{P}_{\mathbf{x}_0}[\textrm{Reach}(S)]=1$. Thus, by proving that $(X_i)_{i=0}^\infty$ defined above is an $\epsilon$-RSM for $T_S$, we also prove Claim~1. We provide an overview of martingale theory results used in this proof in Appendix~\ref{sec:martingales}.

    \medskip\noindent{\em Claim 2.} $\mathbb{P}_{\mathbf{x}_0}[\exists\, t\in\mathbb{N}_0 \text{ s.t. } \mathbf{x}_t\not\in\Stable] = p < 1$ where $p=\frac{M+L_V\cdot \Delta}{M+L_V\cdot \Delta+\delta}$, for each $\mathbf{x}_0 \in S$.

    \medskip\noindent To prove Claim~2, recall that $S = \{\mathbf{x}\in\mathcal{X} \mid V(\mathbf{x})\leq M\}$. Thus, as $V$ is Lipschitz continuous with Lipschitz constant $L_V$ and $\Delta$ is the maximal step size of the system, it follows that the value of $V$ immediately upon the agent leaving the set $S$ is $\leq M+L_V\cdot \Delta$. Hence, for the agent to leave $\Stable$ from $\mathbf{x}_0\in S$, it first has to reach a state $\mathbf{x}_1$ with $M<V(\mathbf{x}_1)\leq M+L_V\cdot \Delta$ and then to also reach a state $\mathbf{x}_2\not\in\Stable$ from $\mathbf{x}_1$ without reentering $S$. By condition~$3$ in Definition~3 of sRSMs, we  have $V(\mathbf{x}_2)\geq M+L_V\cdot \Delta + \delta$. We claim that this happens with probability at most $p = \frac{M+L_V\cdot \Delta}{M+L_V\cdot \Delta+\delta}$. To prove this, we use another result from martingale theory which says that, if $(Z_i)_{i=0}^\infty$ is a nonnegative supermartingale and $\lambda>0$, then $\mathbb{P}[\sup_{i\geq 0}Z_i \geq \lambda]\leq\frac{\mathbb{E}[Z_0]}{\lambda}$ (see Appendix~\ref{sec:martingales}). We apply this theorem to the process $(X_i')_{i=0}^\infty$ defined analogously as in eq.~\ref{eq:processmain}, but in the probability space of trajectories that start in $\mathbf{x}_1$. Then, since in this probability space we have that $X_0$ is equal to $V(\mathbf{x}_1) \leq M+L_V\cdot \Delta$, by plugging in $\lambda = M+L_V\cdot \Delta+\delta$ we conclude that the probability of the process ever leaving $\Stable$ and thus reaching a state in which $V \geq M+L_V\cdot \Delta+\delta$ is
    \begin{equation*}
    \begin{split}
    	&\mathbb{P}_{\mathbf{x}_0}[\exists\, t\in\mathbb{N}_0 \text{ s.t. } \mathbf{x}_t\not\in\Stable] \\
    	\leq &\mathbb{P}_{\mathbf{x}_0}[\sup_{i\geq 0}X_i \geq M+L_V\cdot \Delta+\delta] \\
    	\leq &\mathbb{P}_{\mathbf{x}_1}[\sup_{i\geq 0}X_i' \geq M+L_V\cdot \Delta+\delta] \\
    	\leq &\frac{M+L_V\cdot \Delta}{M+L_V\cdot \Delta+\delta} = p < 1,
    \end{split}
    \end{equation*}
    so Claim~2 follows. The above inequality is formally proved in Appendix~\ref{sec:thmproofs}.
    
    \medskip\noindent{\em Claim~1 and Claim~2 imply Theorem~\ref{thm:soundness}.}  Finally, we show that these two claims imply the theorem statement. By Claim~1, the agent with probability $1$ converges to and reaches $S\subseteq\Stable$ from any initial state. On the other hand, by Claim~2, upon reaching a state in $S$ the probability of leaving $\Stable$ is at most $p<1$. Furthermore, even if $\Stable$ is left, by Claim~1 the agent is guaranteed to again converge to and reach $S$. Hence, due to the system dynamics under a fixed policy satisfying Markov property, the probability of the agent leaving and reentering $S$ more than $N$ times is bounded from above by $p^N$. By letting $N\rightarrow \infty$, we conclude that the probability of the agent leaving $\Stable$ and reentering infinitely many times is $0$, so the agent with probability~$1$ eventually enters and $S$ and does not leave $\Stable$ after that. This implies that $\Stable$ is a.s.~asymptotically stable.\hfill\qed
\end{proof}

\paragraph{Bounds on stabilization time} We conclude this section by showing that our sRSMs not only certify a.s.~asymptotic stability of $\Stable$, but also provide bounds on the number of time steps that the agent may spend outside of $\Stable$. This is particularly relevant for safety-critical applications in which the goal is not only to ensure stabilization but also to ensure that the agent spends as little time outside the stabilization set as possible.
For each trajectory $\rho=(\mathbf{x}_t,\mathbf{u}_t,\omega_t)_{t\in\mathbb{N}_0}$, let $\mathsf{Out}_{\Stable}(\rho) = |\{t\in\mathbb{N}_0 \mid \mathbf{x}_t\not\in\Stable \}|\in\mathbb{N}_0\cup\{\infty\}$. 

\begin{theorem}[Proof in Appendix~\ref{sec:thmproofs}]\label{thm:bound}
Let $\epsilon, M, \delta > 0$ and suppose that $V: \mathcal{X} \rightarrow \mathbb{R}$ is an $(\epsilon,M,\delta)$-sRSM for $\Stable$. Let $\Gamma = \sup_{\mathbf{x}\in\Stable}V(\mathbf{x})$ be the supremum of all possible values that $V$ can attain over the stabilizing set $\Stable$. Then, for each initial state $\mathbf{x}_0\in\mathcal{X}$, we have that
\begin{compactenum}
    \item $\mathbb{E}_{\mathbf{x}_0}[\mathsf{Out}_{\Stable}] \leq \frac{V(\mathbf{x}_0)}{\epsilon} + \frac{(M+L_V\cdot \Delta)\cdot (\Gamma + L_V\cdot \Delta)}{\delta\cdot\epsilon}$.
    \item $\mathbb{P}_{\mathbf{x}_0}[\mathsf{Out}_{\Stable} \geq t] \leq \frac{V(\mathbf{x}_0)}{t\cdot\epsilon} + \frac{(M+L_V\cdot \Delta)\cdot (\Gamma + L_V\cdot \Delta)}{\delta\cdot\epsilon\cdot t}$, for any time $t\in\mathbb{N}$.
\end{compactenum}
\end{theorem}


\section{Learning Stabilizing Policies and sRSMs on Compact State Spaces}\label{sec:algo}
In this section, we present our method for learning a stabilizing policy together with an sRSM that formally certifies a.s.~asymptotic stability. As stated in Section~\ref{sec:prelims}, our method assumes that the state space $\mathcal{X}\subseteq\mathbb{R}^n$ is compact and that $f$ is Lipschitz continuous with Lipschitz constant $L_f$. We prove that, if the method outputs a policy, then it guarantees a.s.~asymptotic stability. After presenting the method for learning control policies, we show that it can also be adapted to a formal verification procedure that learns an sRSM for a given Lipschitz continuous control policy $\pi$.

\paragraph{Outline of the method} We parameterize the policy and the sRSM via two neural networks $\pi_\theta: \mathcal{X} \rightarrow \mathcal{U}$ and $V_\nu: \mathcal{X} \rightarrow \mathbb{R}$, where $\theta$ and $\nu$ are vectors of neural network parameters. To enforce condition 1 in Definition~\ref{def:stabilizingrsm}, which requires the sRSM to be a nonnegative function, our method applies the softplus activation function $x\mapsto \log(\exp(x)+1)$ to the output of $V_\nu$. The remaining layers of $\pi_{\theta}$ and $V_{\nu}$ apply ReLU activation functions, therefore $\pi_{\theta}$ and $V_{\nu}$ are also Lipschitz continuous~\cite{SzegedyZSBEGF13}. Our method draws insight from the algorithms of~\cite{ChangRG19,ZikelicLCH22} for learning policies together with Lyapunov functions or RSMs and it comprises of a {\em learner} and a {\em verifier} module that are composed into a loop. In each loop iteration, the learner module first trains both $\pi_\theta$ and $V_\nu$ on a training objective in the form of a differentiable approximation of the sRSM conditions 2 and 3 in Definition \ref{def:stabilizingrsm}. Once the training has converged, the verifier module formally checks whether the learned sRSM candidate satisfies conditions 2 and 3 in Definition \ref{def:stabilizingrsm}.
If both conditions are fulfilled, our method terminates and returns a policy together with an sRSM that formally certifies a.s.~asymptotic stability.
If at least one sRSM condition is violated, the verifier module enlarges the training set of the learner by counterexample states that violate the condition in order to guide the learner towards fixing the policy and the sRSM in the next learner iteration. This loop is repeated until either the verifier successfully verifies the learned sRSM and outputs the control policy and the sRSM, or until some specified timeout is reached in which case no control policy is returned by the method. The pseudocode of the algorithm is shown in Algorithm~\ref{alg:algorithm}. In what follows, we provide details on algorithm initialization (lines 3-6, Algorithm~\ref{alg:algorithm}) and on the learner and the verifier modules (lines 7-22, Algorithm~\ref{alg:algorithm}). 

\begin{algorithm}[t]
\caption{Learner-verifier procedure}
\label{alg:algorithm}
\begin{algorithmic}[1]
\STATE \textbf{Input} Dynamics function $f$, stochastic disturbance distribution $d$, stabilizing region $\Stable\subseteq\mathcal{X}$, Lipschitz constant~$L_f$
\STATE \textbf{Parameters} $\tau>0$, $N_{\text{cond~2}}\in\mathbb{N}$, $N_{\text{cond~3}}\in\mathbb{N}$, $\epsilon_{\text{train}}$, $\delta_{\text{train}}$ \\ 
\STATE $\tilde{\mathcal{X}} \leftarrow $ centers of cells of a discretization rectangular grid in $\mathcal{X}$ with mesh $\tau$
\STATE $B \leftarrow$ centers of grid cells of a subgrid of $\tilde{\mathcal{X}}$
\STATE $\pi_{\theta} \leftarrow $  policy trained by using PPO \cite{schulman2017proximal} 
\STATE $M \leftarrow 1$
\WHILE{timeout not reached}
\STATE $\pi_{\theta}, V_{\nu} \leftarrow$ jointly trained by minimizing the loss in \eqref{eq:loss} on dataset $B$
\STATE $\tilde{\mathcal{X}}_{\geq M} \leftarrow$ centers of cells over which $V_{\nu}(\mathbf{x})\geq M$ 
\STATE $L_\pi, L_V \leftarrow$ Lipschitz constants of $\pi_{\theta}$, $V_{\nu}$
\STATE $K\leftarrow L_V \cdot (L_f \cdot (L_\pi + 1) + 1)$
\STATE $\tilde{\mathcal{X}}_{ce} \leftarrow $ counterexamples to condition~2 on $\tilde{\mathcal{X}}_{\geq M}$
\IF{$\tilde{\mathcal{X}}_{ce} = \{\}$}
\STATE $\text{Cells}_{\mathcal{X}\backslash\Stable} \leftarrow$ grid cells that intersect $\mathcal{X}\backslash\Stable$
\STATE $\Delta_{\theta} \leftarrow$ max.~step size of the system with policy $\pi$
\IF{$\underline{V}_{\nu}(\text{cell}) > M+L_V\cdot \Delta_{\theta}$ for all $\text{cell}\in\text{Cells}_{\mathcal{X}\backslash\Stable}$}
\STATE \textbf{return} $\pi_{\theta}$, $V_{\nu}$, ``$\Stable$ is a.s.~asymptotically stable under $\pi_{\theta}$''
\ENDIF
\ELSE 
\STATE $B \leftarrow (B \setminus \{\mathbf{x} \in B | V_{\nu}(\mathbf{x}) < M\}) \cup \tilde{\mathcal{X}}_{ce}$
\ENDIF
\ENDWHILE
\STATE \textbf{Return} Unknown
\end{algorithmic}
\end{algorithm}

\subsection{Initialization}

\paragraph{State space discretization} The key challenge in verifying an sRSM candidate is to check the expected decrease condition imposed by condition~2 in Definition~\ref{def:stabilizingrsm}. To check this condition, following the idea of \cite{BerkenkampTS017} and \cite{lechner2021stability} our method first computes a discretization of the state space $\mathcal{X}$. A {\em discretization} $\tilde{\mathcal{X}}$ of $\mathcal{X}$ with {\em mesh} $\tau>0$ is a finite subset $\tilde{\mathcal{X}}\subseteq\mathcal{X}$ such that for every $\mathbf{x}\in\mathcal{X}$ there exists $\tilde{\mathbf{x}}\in\tilde{\mathcal{X}}$ with $||\tilde{\mathbf{x}}-\mathbf{x}||_1<\tau$. Our method computes the discretization by considering {\em centers of cells of a rectangular grid} of sufficiently small cell size (line~3, Algorithm~\ref{alg:algorithm}). The discretization will later be used by the verifier in order to reduce verification of condition~2 to checking a slightly stricter condition at discretization vertices, due to all involved functions being Lipschitz continuous (more details Section~\ref{sec:verifier}).

The algorithm also collects the set $B$ of grid cell centers of a subgrid of $\tilde{\mathcal{X}}$ of larger mesh (line~4, Algorithm~\ref{alg:algorithm}). This set will be used as the initial training set for the learner, and will then be gradually expanded by counterexamples computed by the verifier.

\paragraph{Policy initialization} We initialize parameters of the neural network policy $\pi_{\theta}$ by running several iterations of the proximal policy optimization (PPO)~\cite{schulman2017proximal} RL algorithm (line~5, Algorithm~\ref{alg:algorithm}). In particular, we induce a Markov decision process (MDP) from the given system by using the reward function $r:\mathcal{X} \rightarrow \mathbb{R}$ defined via
\[ r(\mathbf{x}) = \begin{cases}
    1, &\text{if } \mathbf{x} \in \Stable \\
    0, &\text{otherwise}
\end{cases} \]
in order to learn an initial policy that drives the system toward the stabilizing set. The practical importance of initialization for learning stabilizing policies in deterministic systems was observed in~\cite{ChangRG19}. 

\paragraph{Fix the value $M=1$} As the last initialization step, we observe that one may always rescale the value of an sRSM by a strictly positive constant factor while preserving all conditions in Definition~\ref{def:stabilizingrsm}. Therefore, without loss of generality, we fix the value $M=1$ in Definition~\ref{def:stabilizingrsm} for our sRSM (line~6, Algorithm~\ref{alg:algorithm}).

\subsection{Learner}\label{sec:learner}
The policy and the sRSM candidate are learned by minimizing the loss
\begin{equation}\label{eq:loss}
    \loss(\theta, \nu) = \loss_{\text{cond~2}}(\theta, \nu) + \loss_{\text{cond~3}}(\theta, \nu)
\end{equation}
(line~8, Algorithm~\ref{alg:algorithm}). The two loss terms guide the learner towards an sRSM candidate that satisfies conditions~2 and~3 in Definition~\ref{def:stabilizingrsm}.

We define the loss term for condition~2 via
\begin{equation*}
\begin{split}
    &\loss_{\text{cond~2}}(\theta, \nu) = \frac{1}{|B|}\sum_{\mathbf{x}\in B}\Big( \max\Big\{ \\
    &\sum_{\omega_1,\dots, \omega_{N_{\text{cond~2}}} \sim d}\frac{V_{\nu}\big(f(\mathbf{x},\pi_\theta(\mathbf{x}),\omega_i)\big)}{N_{\text{cond~2}}} -  V_{\nu}(\mathbf{x})  + \epsilon_{\text{train}}, 0\Big\} \Big).
\end{split}
\end{equation*}
Intuitively, for each element $\mathbf{x} \in B$ of the training set, the corresponding term in the sum incurs a loss whenever condition~2 is violated at $\mathbf{x}$. Since the expected value of $V_{\nu}$ at a successor state of $\mathbf{x}$ does not admit a closed form expression due to $V_{\nu}$ being a neural network, we approximate it as the mean of values of $V_{\nu}$ at $N_{\text{cond~2}}$ independently sampled successor states of $\mathbf{x}$, with $N_{\text{cond~2}}$ being an algorithm parameter.

For condition~3, the loss term samples $N_{\text{cond~3}}$ system states from $ \mathcal{X}\backslash\Stable$ with $N_{\text{cond~3}}$ an algorithm parameter and incurs a loss whenever condition~3 is not satisfied at some sampled state:
\begin{equation*}
\begin{split}
    \loss_{\text{cond3}}(\theta, \nu) = \max\Big\{M + L_{V_\nu} + \Delta_\theta + \delta_{\text{train}} - \min_{x_1,\dots x_{N_{\text{cond~3}}} \sim  \mathcal{X} \backslash \mathcal{X}_s}V_\nu(x_i), 0\Big\}.
\end{split}
\end{equation*}

\paragraph{Regularization terms in the implementation} In our implementation, we also add two regularization terms to the loss function used by the learner. The first term favors learning an sRSM candidate whose global minimum is within the stabilizing set. The second term penalizes large Lipschitz bounds of the networks $\pi_\theta$ and $V_\nu$ by adding a regularization term. While these two loss terms do not directly enforce any particular condition in Definition~\ref{def:stabilizingrsm}, we observe that they help the learning and the verification process and decrease the number of needed learner-verifier iterations. See Appendix~\ref{sec:regularization} for details on regularization terms.

\subsection{Verifier}\label{sec:verifier}

The verifier formally checks whether the learned sRSM candidate satisfies conditions~2 and~3 in Definition~\ref{def:stabilizingrsm}. Recall, condition~1 is satisfied due to the softplus activation function applied to the output of $V_{\nu}$. 

\paragraph{Formal verification of condition~2} The key challenge in checking the expected decrease condition in condition~2 is that the expected value of a neural network function does not admit a closed-form expression, so we cannot evaluate it directly. Instead, we check condition~2 by first showing that it suffices to check a slightly stricter condition at vertices of the discretization $\tilde{\mathcal{X}}$, due to all involved functions being Lipschitz continuous. We then show how this stricter condition is checked at each discretization vertex.

To verify condition~2, the verifier first collects the set $\tilde{\mathcal{X}}_{\geq M}$ of centers of all grid cells that contain a state $\mathbf{x}$ with $V_{\nu}(\mathbf{x})\geq M$ (line~9, Algorithm~\ref{alg:algorithm}). This set is computed via interval arithmetic abstract interpretation (IA-AI)~\cite{CousotC77,Gowal18}, which for each grid cell propagates interval bounds across neural network layers in order to bound from below the minimal value that $V_{\nu}$ attains over that cell. The center of the grid cell is added to $\tilde{\mathcal{X}}_{\geq M}$ whenever this lower bound is smaller than $M$. We use the method of~\cite{Gowal18} to perform IA-AI with respect to a neural network function $V_\nu$ so we refer the reader to~\cite{Gowal18} for details on this step. 

Once $\tilde{\mathcal{X}}_{\geq M}$ is computed, the verifier uses the method of~\cite[Section~4.3]{SzegedyZSBEGF13} to compute the Lipschitz constants $L_\pi$ and $L_V$ of neural networks $\pi_{\theta}$ and $V_{\nu}$, respectively (line~10, Algorithm~\ref{alg:algorithm}). It then sets $K=L_V \cdot (L_f \cdot (L_\pi + 1) + 1)$ (line~11, Algorithm~\ref{alg:algorithm}). Finally, for each $\tilde{\mathbf{x}} \in \tilde{\mathcal{X}}_{\geq M}$ the verifier checks the following stricter inequality
\begin{equation}\label{eq:expdecstricter}
   \mathbb{E}_{\omega\sim d}\Big[ V_{\nu} \Big( f(\tilde{\mathbf{x}}, \pi_{\theta}(\tilde{\mathbf{x}}), \omega) \Big) \Big] < V_{\nu}(\tilde{\mathbf{x}}) - \tau \cdot K,
\end{equation}
and collects the set $\tilde{\mathcal{X}}_{ce} \subseteq \tilde{\mathcal{X}}_{\geq M}$ of counterexamples at which this inequality is violated (line~12, Algorithm~\ref{alg:algorithm}).
The reason behind checking this stronger constraint is that, due to Lipschitz continuity of all involved functions and due to $\tau$ being the mesh of the discretization, we can show (formally done in the proof of Theorem~\ref{thm:verifiermain}) that this condition being satisfied for each $\tilde{\mathbf{x}} \in \tilde{\mathcal{X}}_{\geq M}$ implies that the expected decrease condition $\mathbb{E}_{\omega\sim d}[ V_{\nu} ( f(\mathbf{x}, \pi_{\theta}(\tilde{\mathbf{x}}), \omega) )] < V_{\nu}(\mathbf{x})$ is satisfied for all $\mathbf{x}\in \mathcal{X}$ with $V(\mathbf{x})\geq M$. Then, due to both sides of the inequality being continuous functions and $\{\mathbf{x}\in\mathcal{X} \mid V_{\nu}(\mathbf{x})\geq M\}$ being a compact set, their difference admits a strictly positive global minimum $\epsilon>0$ so that $\mathbb{E}_{\omega\sim d}[ V_{\nu} ( f(\mathbf{x}, \pi_{\theta}(\tilde{\mathbf{x}}), \omega) )] \leq V_{\nu}(\mathbf{x}) - \epsilon$ is satisfied for all $\mathbf{x}\in \mathcal{X}$ with $V(\mathbf{x})\geq M$. We show in the paragraph below how our method formally checks whether the inequality in \eqref{eq:expdecstricter} is satisfied at some $\tilde{\mathbf{x}} \in \tilde{\mathcal{X}}_{\geq M}$.

If \eqref{eq:expdecstricter} is satisfied for each $\tilde{\mathbf{x}} \in \tilde{\mathcal{X}}_{\geq M}$ and so $\tilde{\mathcal{X}}_{ce} = \emptyset$, the verifier concludes that $V_{\nu}$ satisfies condition~2 in Definition~\ref{def:stabilizingrsm} and proceeds to checking condition~3 in Definition~\ref{def:stabilizingrsm} (lines 14-18, Algorithm~\ref{alg:algorithm}). Otherwise, any computed counterexample to this constraint is added to $B$ to help the learner fine-tune an sRSM candidate (line 20, Algorithm~\ref{alg:algorithm}) and the algorithm proceeds to the start of the next learner-verifer iteration (line 7, Algorithm~\ref{alg:algorithm}).

\paragraph{Checking inequality \eqref{eq:expdecstricter} and expected value computation} To check \eqref{eq:expdecstricter} at some $\tilde{\mathbf{x}} \in \tilde{\mathcal{X}}_{\geq M}$, we need to compute the expected value $\mathbb{E}_{\omega\sim d}[ V_{\nu} ( f(\tilde{\mathbf{x}}, \pi_{\theta}(\tilde{\mathbf{x}}), \omega) ) ]$. Note that this expected value does not admit a closed form expression due to $V_{\nu}$ being a neural network function, so we cannot evaluate it directly. Instead, we use the method of~\cite{lechner2021stability} in order to compute an upper bound on this expected value and use this upper bound to formally check whether \eqref{eq:expdecstricter} is satisfied at $\tilde{\mathbf{x}}$. For completeness of our presentation, we briefly describe this expected value bound computation below.
Recall, in our assumptions in Section~\ref{sec:prelims}, we said that our algorithm assumes that the stochastic disturbance space $\mathcal{N}$ is bounded or that $d$ is a product of independent univariate distributions.

First, consider the case when $\mathcal{N}$ is bounded. We partition the stochastic disturbance space $\mathcal{N}\subseteq\mathbb{R}^p$ into finitely many cells $\text{cell}(\mathcal{N}) = \{\mathcal{N}_1,\dots,\mathcal{N}_{k}\}$. We denote by $\mathrm{maxvol}=\max_{\mathcal{N}_i\in \text{cell}(\mathcal{N})}\mathsf{vol}(\mathcal{N}_i)$ the maximal volume of any cell in the partition with respect to the Lebesgue measure over $\mathbb{R}^p$. The expected value can then be bounded from above via
\begin{equation*}
     \mathbb{E}_{\omega\sim d}\Big[ V_{\nu} \Big( f(\tilde{\mathbf{x}}, \pi_{\theta}(\tilde{\mathbf{x}}), \omega) \Big) \Big] \leq \sum_{\mathcal{N}_i\in \text{cell}(\mathcal{N})} \mathrm{maxvol} \cdot \sup_{\omega\in \mathcal{N}_i} F(\omega)
 \end{equation*}
 where $F(\omega) = V_{\nu} ( f(\tilde{\mathbf{x}}, \pi_{\theta}(\tilde{\mathbf{x}}), \omega)$. Each supremum on the right-hand-side is then bounded from above by using the IA-AI-based method of ~\cite{Gowal18}. 
 
Second, consider the case when $\mathcal{N}$ is unbounded but $d$ is a product of independent univariate distributions. Note that in this case we cannot directly follow the above approach since $\mathrm{maxvol}=\max_{\mathcal{N}_i\in \text{cell}(\mathcal{N})}\mathsf{vol}(\mathcal{N}_i)$ would be infinite. However, since $d$ is a product of independent univariate distributions, we may first apply the Probability Integral Transform~\cite{Murphy12} to each univariate distribution in $d$ to transform it into a finite support distribution and then proceed as above.

\paragraph{Formal verification of condition~3} To formally verify condition~3 in Definition~\ref{def:stabilizingrsm}, the verifier collects the set $\text{Cells}_{\mathcal{X}\backslash\Stable}$ of all grid cells that intersect $\mathcal{X}\backslash\Stable$ (line 14, Algorithm~\ref{alg:algorithm}). Then, for each $\text{cell}\in\text{Cells}_{\mathcal{X}\backslash\Stable}$, it uses IA-AI to check
\begin{equation}\label{eq:stablecond}
\underline{V}_{\nu}(\text{cell}) > M+L_V\cdot \Delta_{\theta},
\end{equation}
with $\underline{V}_{\nu}(\text{cell})$ denoting the lower bound on $V_{\nu}$ over the $\text{cell}$ computed by IA-AI (lines 15-16, Algorithm~\ref{alg:algorithm}). If this holds, then the verifier concludes that $V_{\nu}$ satisfies condition~3 in Definition~\ref{def:stabilizingrsm} with $\delta=\min_{\text{cell}\in\text{Cells}_{\mathcal{X}\backslash\Stable}}\{\underline{V}_{\nu}(\text{cell}) - M - L_V\cdot \Delta_{\theta}\}$. Hence, as conditions~2 and~3 have both been formally verified to be satisfied, the method returns the policy $\pi_\theta$ and the sRSM $V_\nu$ which formally proves that $\Stable$ is a.s.~asymptotically stable under $\pi_\theta$ (line~17, Algorithm~\ref{alg:algorithm}). Otherwise, the method proceeds to the next learner-verifier loop iteration (line 7, Algorithm~\ref{alg:algorithm}).

\paragraph{Algorithm correctness} The following theorem establishes the correctness of Algorithm~\ref{alg:algorithm}. In particular, it shows that if the verifier confirms that conditions~2 and~3 in Definition~\ref{def:stabilizingrsm} are satisfied and therefore Algorithm~\ref{alg:algorithm} returns a control policy $\pi_\theta$ and an sRSM $V_\nu$, then it holds that $V_\nu$ is indeed an sRSM and that $\Stable$ is a.s.~asymptotically stable under $\pi_\theta$.

\begin{theorem}[Algorithm correctness, proof in Appendix~\ref{app:thmthreeproof}]\label{thm:verifiermain}
Suppose that the verifier shows that $V_{\nu}$ satisfies \eqref{eq:expdecstricter} for each $\tilde{\mathbf{x}} \in \tilde{\mathcal{X}}_{\geq M}$ and \eqref{eq:stablecond} for each $\text{cell}\in\text{Cells}_{\mathcal{X}\backslash\Stable}$, so Algorithm~\ref{alg:algorithm} returns $\pi_\theta$ and $V_{\nu}$. Then $V_{\nu}$ is an sRSM and $\Stable$ is a.s.~asymptotically stable under~$\pi_{\theta}$.
\end{theorem}

\subsection{Adaptation into a formal verification procedure}\label{sec:verificationadaptation}

To conclude this section, we show that Algorithm~\ref{alg:algorithm} can be easily adapted into a formal verification procedure for showing that $\Stable$ is a.s.~asymptotically stable under some given control policy $\pi$. This adaptation only assumes that $\pi$ is Lipschitz continuous with a given Lipschitz constant $L_\pi$, or alternatively that it is a neural network policy with Lipschitz continuous activation functions in which case we use the method of~\cite{SzegedyZSBEGF13} to compute its Lipschitz constant $L_\pi$.

Instead of jointly learning the control policy and the sRSM, the formal verification procedure now only learns a neural network sRSM $V_\nu$. This is done by executing the analogous learner-verifier loop described in Algorithm~\ref{alg:algorithm}. The only difference happens in the learner module, where now only the parameters $\nu$ of the sRSM neural network are learned. Hence, the loss function in \eqref{eq:loss} that is used in (line~8, Algorithm~\ref{alg:algorithm}) has the same form as in Section~\ref{sec:learner}, but now it only takes parameters $\nu$ as input:
\begin{equation*}
    \loss(\nu) = \loss_{\text{cond~2}}(\nu) + \loss_{\text{cond~3}}(\nu).
\end{equation*}
Additionally, the control policy initialization in (line~5, Algorithm~\ref{alg:algorithm}) becomes redundant because the control policy $\pi$ is given. Apart from these two changes, the formal verification procedure remains identical to Algorithm~\ref{alg:algorithm} and its correctness follows from Theorem~\ref{thm:verifiermain}.




\begin{figure}[t]
    \centering
    \includegraphics[width=0.8\linewidth]{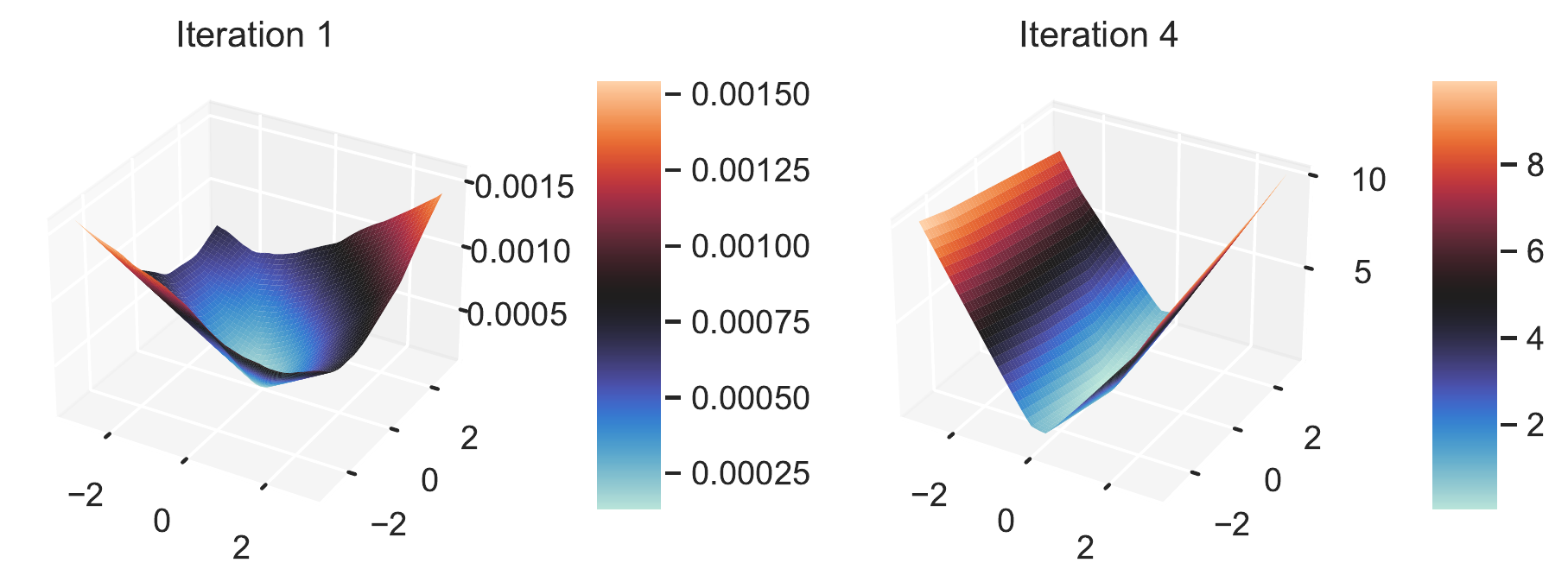}
    \caption{Visualization of the sRSM candidate after 1 and 4 iterations of our algorithm for the inverted pendulum task. The candidate after 1 iteration does not satisfy all sRSM conditions, while the candidate after 4 iterations is an sRSM. }
    \label{fig:pend1}
\end{figure}

\begin{figure}[t]
    \centering
    \includegraphics[width=0.8\linewidth]{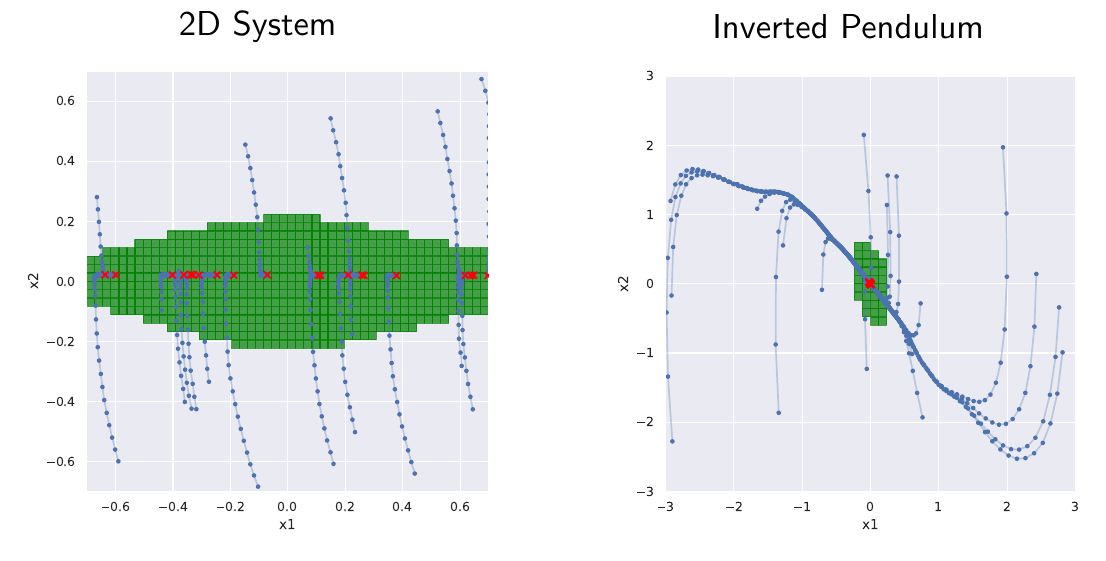}
    \caption{Visualization of the learned stabilizing sets in green, in which the system will remain with probability~1.}
    \label{fig:kernels}
\end{figure}

\section{Experimental Results}\label{sec:experiments}
In this section, we experimentally evaluate the effectiveness of our method\footnote{Our implementation is available at \url{https://github.com/mlech26l/neural_martingales/tree/ATVA2023}}.
We consider the same experimental setting and the two benchmarks studied in~\cite{lechner2021stability}. However, in contrast to~\cite{lechner2021stability}, we do not assume that the stabilization sets are closed under system dynamics and that the system stabilizes immediately upon reaching the stabilization set. In our evaluation, we modify both environments so that this assumption is violated. The goal of our evaluation is to confirm that our method based on sRSMs can in practice learn policies that formally guarantee a.s.~asymptotic stability even when the stabilization set is not closed under system dynamics.

We parameterize both $\pi_\theta$ and $V_\nu$ by two fully-connected neural networks with 2 hidden ReLU layers, each with 128 neurons. Below we describe both benchmarks considered in our evaluation, and refer the reader to Appendix~\ref{app:experimentsdetails} for further details and formal definitions of environment dynamics.

The first benchmark is a two-dimensional linear dynamical system with non-linear control bounds and is of the form $x_{t+1} = Ax_t + Bg(u_t) + \omega$, where $\omega$ is a stochastic disturbance vector sampled from a zero-mean triangular distribution. The function $g$ clips the action to stay within the interval [1, -1].
The state space is $\mathcal{X} = \{x \mid |x_1| \le 0.7, |x_2| \le 0.7\}$ and we want to learn a policy for the stabilizing~set
\begin{eqnarray*}
\begin{split}
    \mathcal{X}_s = \mathcal{X} \backslash \Big( &\{x \mid -0.7 \le x_1 \le -0.6, -0.7\le x_2 \le -0.4\} \\
    &\bigcup \{x \mid 0.6 \le x_1 \le 0.7, 0.4 \le x_2 \le 0.7\} \Big).
\end{split}
\end{eqnarray*}

\begin{table}[t]
    \centering
    \begin{tabular}{c|c c c}
      \toprule
      Benchmark & Iters. & Mesh ($\tau$) & Runtime\\
      \midrule
      2D system & 5 & 0.0007 & 3660 s\\
      Pendulum & 4  & 0.003 & 2619 s\\
      \bottomrule
    \end{tabular}
    \caption{Results of our experimental evaluation. The first column shows benchmark names. The second column shows the numer of learner-verifier loop iterations needed to successfully learn and verify a control policy and an sRSM. The third column shows the mesh of the used discretization grid. The fourth column shows runtime in seconds.}
        \label{tab:table1}
\end{table}


The second benchmark is a modified version of the inverted pendulum problem adapted from the OpenAI gym \cite{brockman2016openai}. Note that this benchmark has non-polynomial dynamics, as its dynamics function involves a sine function (see Appendix~\ref{app:experimentsdetails}).
The system is expressed by two state variables that represent the angle and the angular velocity of the pendulum.  Contrary to the original task, the problem considered here introduces triangular-shaped random noise to the state after each update step.
The state space is define as $\mathcal{X} = \{x \mid |x_1| \le 3, |x_2| \le 3\}$, and objective of the agent is to stabilize the pendulum within the stabilizing set
\begin{eqnarray*}
\begin{split}
    \mathcal{X}_s = \mathcal{X} \backslash \Big( &\{x \mid -3 \le x_1 \le -2.9, -3\le x_2 \le 0\} \\
    &\bigcup \{x \mid 2.9 \le x_1 \le 3, 0 \le x_2 \le 3\} \Big).
\end{split}
\end{eqnarray*}
For both tasks, our algorithm could find valid sRSMs and prove stability. The runtime characteristics, such as the number of iterations and total runtime, is shown in Table~\ref{tab:table1}. In Figure~\ref{fig:pend1} we plot the sRSM found by our algorithm for the inverted pendulum task. We also visualize for both tasks in Figure~\ref{fig:kernels} in green the subset of $\Stable$ implied by the learned sRSM in which the system stabilizes. Finally, in Figure~\ref{fig:contours} we show the contour lines of the expected stabilization time bounds that are obtained by applying Theorem~\ref{thm:bound} to the learned sRSMs.

\begin{figure}[t]
\centering
\includegraphics[width=0.7\linewidth]{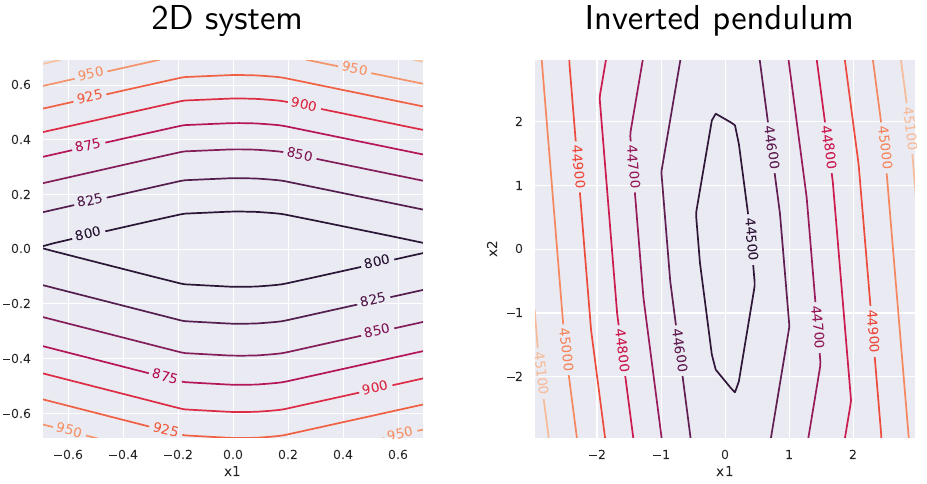}
\caption{Contour lines of the expected stabilization time implied by Theorem~\ref{thm:bound} for the 2D system task on the left and the inverted pendulum task on the right.}
\label{fig:contours}
\end{figure}

\paragraph{Limitations} We conclude by discussing limitations of our appraoch. Verification of neural networks is inherently a computationally difficult problem~\cite{katz2017reluplex,BerkenkampTS017,salzer2021reachability}. Our method is subject to this barrier as well. In particular, the complexity of the grid decomposition routine for checking the expected decrease condition is exponential in the dimension of the system state space. Consideration of different grid decomposition strategies and in particular non-uniform grids that incorporate properties of the state space is an interesting direction of future work towards improving the scalability of our method.
However, a key advantage of our approach is that the complexity is only linear in the size of the neural network policy. Consequently, our approach allows learning and verifying networks that are of the size of typical networks used in reinforcement learning \cite{schulman2017proximal}. 
Moreover, our grid decomposition procedure runs entirely on accelerator devices, including CPUs, GPUs, and TPUs, thus leveraging future advances in these computing devices.
A technical limitation of our learning procedure is that it is restricted to compact state spaces. Our theoretical results are applicable to arbitrary (potentially unbounded) state spaces, as shown in Fig.~\ref{fig:srsmexample}.


\section{Related Work}
\label{sec:relatedwork}

\paragraph{Stability for deterministic systems} Most early works on control with stability constraints rely either on hand-designed certificates or their computation via sum-of-squares (SOS) programming~\cite{henrion2005positive,parrilo2000structured}. Automation via SOS programming is restricted to problems with polynomial dynamics and does not scale well with dimension. Learning-based methods present a promising approach to overcome these limitations~\cite{RichardsB018,Jin20,ChangG21}. In particular, the methods of~\cite{ChangRG19,AbateAGP21} also learn a control policy and a Lyapunov function as neural networks by using a learner-verifier framework that our method builds on and extends to stochastic systems.

\paragraph{Stability for stochastic systems} While the theory behind stochastic system stability is well studied~\cite{kushner1965stability,Kushner14}, only a few works consider automated controller synthesis with formal stability guarantees for stochastic systems with continuous dynamics. The methods of~\cite{CrespoS03,Vaidya15} are numerical and certify weaker notions of stability. Recently,~\cite{lechner2021stability,ZikelicLCH22} used RSMs and learn a stabilizing policy together with an RSM that certifies a.s.~asymptotic stability. However, this method assumes closedness under system dynamics and essentially considers the stability problem as a reachability problem.
In contrast, our proof in Section \ref{sec:theory} introduces a new type of reasoning about supermartingales which allows us to handle stabilization without prior knowledge of a set that is closed under the system dynamics.

\paragraph{Reachability and safety for stochastic systems} Comparatively more works have studied controller synthesis in stochastic systems with formal reachability and safety guarantees. A number of methods abstract the system as a finite-state Markov decision process (MDP) and synthesize a controller for the MDP to provide formal reachability or safety guarantees over finite time horizon~\cite{SoudjaniGA15,LavaeiKSZ20,cauchi2019stochy,VinodGO19}. An abstraction based method for obtaining infinite time horizon PAC-style guarantees on the probability of reach-avoidance in linear stochastic systems was proposed in~\cite{BadingsRAPPSJ23}. A method for formal controller synthesis in infinite time horizon non-linear stochastic systems with guarantees on the probability of co-safety properties was proposed in~\cite{HBSSH23}. A learning-based approach for learning a control policy that provides formal reachability and avoidance infinite time horizon guarantees was proposed in~\cite{ZikelicLHC23}.


\paragraph{Safe exploration RL} Safe exploration RL restricts exploration of RL algorithms in a way that a given safety constraint is satisfied. This is typically ensured by learning the system dynamics' uncertainty and limiting exploratory actions within a high probability safe region via Gaussian Processes~\cite{Koller2018LearningBasedMP,Turchetta2019SafeEF}, linearized models~\cite{Dalal2018SafeEI}, deep robust regression~\cite{Liu2020RobustRF} and Bayesian neural networks~\cite{lechner2021infinite}.

\paragraph{Probabilistic program analysis} Ranking supermartingales were originally proposed for proving a.s.~termination in probabilistic programs (PPs)~\cite{ChakarovS13}. Since then, martingale-based methods have been used for termination~\cite{ChatterjeeFNH16,ChatterjeeFG16,AbateGR20,ChatterjeeGNZZ21} safety~\cite{ChatterjeeNZ17,TakisakaOUH21,ChatterjeeGMZ22} and recurrence and persistence~\cite{ChakarovVS16} analysis in PPs, with the latter being equivalent to stability. However, the persistence certificate of~\cite{ChakarovVS16} is substantially different from ours. In particular, the certificate of~\cite{ChakarovVS16} requires strict expected decrease outside the stabilizing set and non-strict expected decrease within the stabilizing set. In contrast, our sRSMs require strict expected decrease outside and only within a small part of the stabilizing set (see Definition~\ref{def:stabilizingrsm}). We also note that the certificate of~\cite{ChakarovVS16} cannot be combined with our learner-verifier procedure. Indeed, since our verifier module discretizes the state space and verifies a stricter condition at discretization vertices, if we tried to verify an instance of the certificate of~\cite{ChakarovVS16} then we would be verifying the strict expected decrease condition over the whole state space. But this condition is not satisfiable over compact state spaces, as any continuous function must admit a global minimum.


\section{Conclusion}\label{sec:limitconclusion}

In this work, we developed a method for learning control policies for stochastic systems with formal guarantees about the systems' a.s.~asymptotic stability. Compared to the existing literature, which assumes that the stabilizing set is closed under system dynamics and cannot be left once entered, our approach does not impose this assumption. Our method is based on the novel notion of stabilizing ranking supermartingales (sRSMs) that serve as a formal certificate of a.s.~asymptotic stability. We experimentally showed that our learning procedure is able to learn stabilizing policies and stability certificates in practice.

\bibliographystyle{splncs04}
\bibliography{bibliography}

\begin{thebibliography}{10}
\providecommand{\url}[1]{\texttt{#1}}
\providecommand{\urlprefix}{URL }
\providecommand{\doi}[1]{https://doi.org/#1}

\bibitem{AbateAGP21}
Abate, A., Ahmed, D., Giacobbe, M., Peruffo, A.: Formal synthesis of lyapunov
  neural networks. {IEEE} Control. Syst. Lett.  \textbf{5}(3),  773--778
  (2021), \url{https://doi.org/10.1109/LCSYS.2020.3005328}

\bibitem{AbateGR20}
Abate, A., Giacobbe, M., Roy, D.: Learning probabilistic termination proofs.
  In: Silva, A., Leino, K.R.M. (eds.) Computer Aided Verification - 33rd
  International Conference, {CAV} 2021, Virtual Event, July 20-23, 2021,
  Proceedings, Part {II}. Lecture Notes in Computer Science, vol. 12760, pp.
  3--26. Springer (2021), \url{https://doi.org/10.1007/978-3-030-81688-9\_1}

\bibitem{achiam2017constrained}
Achiam, J., Held, D., Tamar, A., Abbeel, P.: Constrained policy optimization.
  In: International Conference on Machine Learning. pp. 22--31. PMLR (2017)

\bibitem{altman1999constrained}
Altman, E.: Constrained Markov decision processes, vol.~7. CRC Press (1999)

\bibitem{AmodeiOSCSM16}
Amodei, D., Olah, C., Steinhardt, J., Christiano, P.F., Schulman, J.,
  Man{\'{e}}, D.: Concrete problems in {AI} safety. CoRR
  \textbf{abs/1606.06565} (2016), \url{http://arxiv.org/abs/1606.06565}

\bibitem{BadingsRAPPSJ23}
Badings, T.S., Romao, L., Abate, A., Parker, D., Poonawala, H.A., Stoelinga,
  M., Jansen, N.: Robust control for dynamical systems with non-gaussian noise
  via formal abstractions. J. Artif. Intell. Res.  \textbf{76},  341--391
  (2023), \url{https://doi.org/10.1613/jair.1.14253}

\bibitem{BerkenkampTS017}
Berkenkamp, F., Turchetta, M., Schoellig, A.P., Krause, A.: Safe model-based
  reinforcement learning with stability guarantees. In: Guyon, I., von Luxburg,
  U., Bengio, S., Wallach, H.M., Fergus, R., Vishwanathan, S.V.N., Garnett, R.
  (eds.) Advances in Neural Information Processing Systems 30: Annual
  Conference on Neural Information Processing Systems 2017, December 4-9, 2017,
  Long Beach, CA, {USA}. pp. 908--918 (2017),
  \url{https://proceedings.neurips.cc/paper/2017/hash/766ebcd59621e305170616ba3d3dac32-Abstract.html}

\bibitem{brockman2016openai}
Brockman, G., Cheung, V., Pettersson, L., Schneider, J., Schulman, J., Tang,
  J., Zaremba, W.: Openai gym. arXiv preprint arXiv:1606.01540  (2016)

\bibitem{cauchi2019stochy}
Cauchi, N., Abate, A.: Stochy-automated verification and synthesis of
  stochastic processes. In: Proceedings of the 22nd ACM International
  Conference on Hybrid Systems: Computation and Control. pp. 258--259 (2019)

\bibitem{ChakarovS13}
Chakarov, A., Sankaranarayanan, S.: Probabilistic program analysis with
  martingales. In: Sharygina, N., Veith, H. (eds.) Computer Aided Verification
  - 25th International Conference, {CAV} 2013, Saint Petersburg, Russia, July
  13-19, 2013. Proceedings. Lecture Notes in Computer Science, vol.~8044, pp.
  511--526. Springer (2013),
  \url{https://doi.org/10.1007/978-3-642-39799-8\_34}

\bibitem{ChakarovVS16}
Chakarov, A., Voronin, Y., Sankaranarayanan, S.: Deductive proofs of almost
  sure persistence and recurrence properties. In: Chechik, M., Raskin, J.
  (eds.) Tools and Algorithms for the Construction and Analysis of Systems -
  22nd International Conference, {TACAS} 2016, Held as Part of the European
  Joint Conferences on Theory and Practice of Software, {ETAPS} 2016,
  Eindhoven, The Netherlands, April 2-8, 2016, Proceedings. Lecture Notes in
  Computer Science, vol.~9636, pp. 260--279. Springer (2016),
  \url{https://doi.org/10.1007/978-3-662-49674-9\_15}

\bibitem{ChangMP92}
Chang, E.Y., Manna, Z., Pnueli, A.: Characterization of temporal property
  classes. In: Kuich, W. (ed.) Automata, Languages and Programming, 19th
  International Colloquium, ICALP92, Vienna, Austria, July 13-17, 1992,
  Proceedings. Lecture Notes in Computer Science, vol.~623, pp. 474--486.
  Springer (1992), \url{https://doi.org/10.1007/3-540-55719-9\_97}

\bibitem{ChangG21}
Chang, Y., Gao, S.: Stabilizing neural control using self-learned almost
  lyapunov critics. In: {IEEE} International Conference on Robotics and
  Automation, {ICRA} 2021, Xi'an, China, May 30 - June 5, 2021. pp. 1803--1809.
  {IEEE} (2021), \url{https://doi.org/10.1109/ICRA48506.2021.9560886}

\bibitem{ChangRG19}
Chang, Y., Roohi, N., Gao, S.: Neural lyapunov control. In: Wallach, H.M.,
  Larochelle, H., Beygelzimer, A., d'Alch{\'{e}}{-}Buc, F., Fox, E.B., Garnett,
  R. (eds.) Advances in Neural Information Processing Systems 32: Annual
  Conference on Neural Information Processing Systems 2019, NeurIPS 2019,
  December 8-14, 2019, Vancouver, BC, Canada. pp. 3240--3249 (2019),
  \url{https://proceedings.neurips.cc/paper/2019/hash/2647c1dba23bc0e0f9cdf75339e120d2-Abstract.html}

\bibitem{ChatterjeeFG16}
Chatterjee, K., Fu, H., Goharshady, A.K.: Termination analysis of probabilistic
  programs through positivstellensatz's. In: Chaudhuri, S., Farzan, A. (eds.)
  Computer Aided Verification - 28th International Conference, {CAV} 2016,
  Toronto, ON, Canada, July 17-23, 2016, Proceedings, Part {I}. Lecture Notes
  in Computer Science, vol.~9779, pp. 3--22. Springer (2016),
  \url{https://doi.org/10.1007/978-3-319-41528-4\_1}

\bibitem{ChatterjeeFNH16}
Chatterjee, K., Fu, H., Novotn{\'{y}}, P., Hasheminezhad, R.: Algorithmic
  analysis of qualitative and quantitative termination problems for affine
  probabilistic programs. In: Bod{\'{\i}}k, R., Majumdar, R. (eds.) Proceedings
  of the 43rd Annual {ACM} {SIGPLAN-SIGACT} Symposium on Principles of
  Programming Languages, {POPL} 2016, St. Petersburg, FL, USA, January 20 - 22,
  2016. pp. 327--342. {ACM} (2016),
  \url{https://doi.org/10.1145/2837614.2837639}

\bibitem{ChatterjeeGMZ22}
Chatterjee, K., Goharshady, A.K., Meggendorfer, T., Zikelic, D.: Sound and
  complete certificates for quantitative termination analysis of probabilistic
  programs. In: Shoham, S., Vizel, Y. (eds.) Computer Aided Verification - 34th
  International Conference, {CAV} 2022, Haifa, Israel, August 7-10, 2022,
  Proceedings, Part {I}. Lecture Notes in Computer Science, vol. 13371, pp.
  55--78. Springer (2022), \url{https://doi.org/10.1007/978-3-031-13185-1\_4}

\bibitem{ChatterjeeGNZZ21}
Chatterjee, K., Goharshady, E.K., Novotn{\'{y}}, P., Z{\'{a}}rev{\'{u}}cky, J.,
  Zikelic, D.: On lexicographic proof rules for probabilistic termination. In:
  Huisman, M., Pasareanu, C.S., Zhan, N. (eds.) Formal Methods - 24th
  International Symposium, {FM} 2021, Virtual Event, November 20-26, 2021,
  Proceedings. Lecture Notes in Computer Science, vol. 13047, pp. 619--639.
  Springer (2021), \url{https://doi.org/10.1007/978-3-030-90870-6\_33}

\bibitem{ChatterjeeNZ17}
Chatterjee, K., Novotn{\'{y}}, P., Zikelic, D.: Stochastic invariants for
  probabilistic termination. In: Castagna, G., Gordon, A.D. (eds.) Proceedings
  of the 44th {ACM} {SIGPLAN} Symposium on Principles of Programming Languages,
  {POPL} 2017, Paris, France, January 18-20, 2017. pp. 145--160. {ACM} (2017),
  \url{https://doi.org/10.1145/3009837.3009873}

\bibitem{ChowNDG18}
Chow, Y., Nachum, O., Du{\'{e}}{\~{n}}ez{-}Guzm{\'{a}}n, E.A., Ghavamzadeh, M.:
  A lyapunov-based approach to safe reinforcement learning. In: Bengio, S.,
  Wallach, H.M., Larochelle, H., Grauman, K., Cesa{-}Bianchi, N., Garnett, R.
  (eds.) Advances in Neural Information Processing Systems 31: Annual
  Conference on Neural Information Processing Systems 2018, NeurIPS 2018,
  December 3-8, 2018, Montr{\'{e}}al, Canada. pp. 8103--8112 (2018),
  \url{https://proceedings.neurips.cc/paper/2018/hash/4fe5149039b52765bde64beb9f674940-Abstract.html}

\bibitem{CousotC77}
Cousot, P., Cousot, R.: Abstract interpretation: {A} unified lattice model for
  static analysis of programs by construction or approximation of fixpoints.
  In: Graham, R.M., Harrison, M.A., Sethi, R. (eds.) Conference Record of the
  Fourth {ACM} Symposium on Principles of Programming Languages, Los Angeles,
  California, USA, January 1977. pp. 238--252. {ACM} (1977),
  \url{https://doi.org/10.1145/512950.512973}

\bibitem{CrespoS03}
Crespo, L.G., Sun, J.: Stochastic optimal control via bellman's principle.
  Autom.  \textbf{39}(12),  2109--2114 (2003),
  \url{https://doi.org/10.1016/S0005-1098(03)00238-3}

\bibitem{Dalal2018SafeEI}
Dalal, G., Dvijotham, K., Vecer{\'i}k, M., Hester, T., Paduraru, C., Tassa, Y.:
  Safe exploration in continuous action spaces. ArXiv  \textbf{abs/1801.08757}
  (2018)

\bibitem{FioritiH15}
Fioriti, L.M.F., Hermanns, H.: Probabilistic termination: Soundness,
  completeness, and compositionality. In: Rajamani, S.K., Walker, D. (eds.)
  Proceedings of the 42nd Annual {ACM} {SIGPLAN-SIGACT} Symposium on Principles
  of Programming Languages, {POPL} 2015, Mumbai, India, January 15-17, 2015.
  pp. 489--501. {ACM} (2015), \url{https://doi.org/10.1145/2676726.2677001}

\bibitem{GarciaF15}
Garc{\'{\i}}a, J., Fern{\'{a}}ndez, F.: A comprehensive survey on safe
  reinforcement learning. J. Mach. Learn. Res.  \textbf{16},  1437--1480
  (2015), \url{http://dl.acm.org/citation.cfm?id=2886795}

\bibitem{Geibel06}
Geibel, P.: Reinforcement learning for mdps with constraints. In:
  F{\"{u}}rnkranz, J., Scheffer, T., Spiliopoulou, M. (eds.) Machine Learning:
  {ECML} 2006, 17th European Conference on Machine Learning, Berlin, Germany,
  September 18-22, 2006, Proceedings. Lecture Notes in Computer Science,
  vol.~4212, pp. 646--653. Springer (2006),
  \url{https://doi.org/10.1007/11871842\_63}

\bibitem{Gowal18}
Gowal, S., Dvijotham, K., Stanforth, R., Bunel, R., Qin, C., Uesato, J.,
  Arandjelovic, R., Mann, T.A., Kohli, P.: On the effectiveness of interval
  bound propagation for training verifiably robust models. CoRR
  \textbf{abs/1810.12715} (2018), \url{http://arxiv.org/abs/1810.12715}

\bibitem{henrion2005positive}
Henrion, D., Garulli, A.: Positive polynomials in control, vol.~312. Springer
  Science \& Business Media (2005)

\bibitem{Jin20}
Jin, W., Wang, Z., Yang, Z., Mou, S.: Neural certificates for safe control
  policies. CoRR  \textbf{abs/2006.08465} (2020),
  \url{https://arxiv.org/abs/2006.08465}

\bibitem{katz2017reluplex}
Katz, G., Barrett, C., Dill, D.L., Julian, K., Kochenderfer, M.J.: Reluplex: An
  efficient smt solver for verifying deep neural networks. In: International
  conference on computer aided verification. pp. 97--117. Springer (2017)

\bibitem{khalil2002nonlinear}
Khalil, H.: Nonlinear Systems. Pearson Education, Prentice Hall (2002)

\bibitem{Koller2018LearningBasedMP}
Koller, T., Berkenkamp, F., Turchetta, M., Krause, A.: Learning-based model
  predictive control for safe exploration. 2018 IEEE Conference on Decision and
  Control (CDC) pp. 6059--6066 (2018)

\bibitem{kushner1965stability}
Kushner, H.J.: On the stability of stochastic dynamical systems. Proceedings of
  the National Academy of Sciences of the United States of America
  \textbf{53}(1), ~8 (1965)

\bibitem{Kushner14}
Kushner, H.J.: A partial history of the early development of continuous-time
  nonlinear stochastic systems theory. Autom.  \textbf{50}(2),  303--334
  (2014), \url{https://doi.org/10.1016/j.automatica.2013.10.013}

\bibitem{LavaeiKSZ20}
Lavaei, A., Khaled, M., Soudjani, S., Zamani, M.: {AMYTISS:} parallelized
  automated controller synthesis for large-scale stochastic systems. In:
  Lahiri, S.K., Wang, C. (eds.) Computer Aided Verification - 32nd
  International Conference, {CAV} 2020, Los Angeles, CA, USA, July 21-24, 2020,
  Proceedings, Part {II}. Lecture Notes in Computer Science, vol. 12225, pp.
  461--474. Springer (2020),
  \url{https://doi.org/10.1007/978-3-030-53291-8\_24}

\bibitem{lechner2021infinite}
Lechner, M., Zikelic, D., Chatterjee, K., Henzinger, T.A.: Infinite time
  horizon safety of bayesian neural networks. In: Ranzato, M., Beygelzimer, A.,
  Dauphin, Y.N., Liang, P., Vaughan, J.W. (eds.) Advances in Neural Information
  Processing Systems 34: Annual Conference on Neural Information Processing
  Systems 2021, NeurIPS 2021, December 6-14, 2021, virtual. pp. 10171--10185
  (2021),
  \url{https://proceedings.neurips.cc/paper/2021/hash/544defa9fddff50c53b71c43e0da72be-Abstract.html}

\bibitem{lechner2021stability}
Lechner, M., Zikelic, D., Chatterjee, K., Henzinger, T.A.: Stability
  verification in stochastic control systems via neural network
  supermartingales. In: Thirty-Sixth {AAAI} Conference on Artificial
  Intelligence, {AAAI} 2022, Thirty-Fourth Conference on Innovative
  Applications of Artificial Intelligence, {IAAI} 2022, The Twelveth Symposium
  on Educational Advances in Artificial Intelligence, {EAAI} 2022 Virtual
  Event, February 22 - March 1, 2022. pp. 7326--7336. {AAAI} Press (2022),
  \url{https://ojs.aaai.org/index.php/AAAI/article/view/20695}

\bibitem{Liu2020RobustRF}
Liu, A., Shi, G., Chung, S.J., Anandkumar, A., Yue, Y.: Robust regression for
  safe exploration in control. In: L4DC (2020)

\bibitem{Murphy12}
Murphy, K.P.: Machine learning - a probabilistic perspective. Adaptive
  computation and machine learning series, {MIT} Press (2012)

\bibitem{parrilo2000structured}
Parrilo, P.A.: Structured semidefinite programs and semialgebraic geometry
  methods in robustness and optimization. California Institute of Technology
  (2000)

\bibitem{Puterman94}
Puterman, M.L.: Markov Decision Processes: Discrete Stochastic Dynamic
  Programming. Wiley Series in Probability and Statistics, Wiley (1994),
  \url{https://doi.org/10.1002/9780470316887}

\bibitem{RichardsB018}
Richards, S.M., Berkenkamp, F., Krause, A.: The lyapunov neural network:
  Adaptive stability certification for safe learning of dynamical systems. In:
  2nd Annual Conference on Robot Learning, CoRL 2018, Z{\"{u}}rich,
  Switzerland, 29-31 October 2018, Proceedings. Proceedings of Machine Learning
  Research, vol.~87, pp. 466--476. {PMLR} (2018),
  \url{http://proceedings.mlr.press/v87/richards18a.html}

\bibitem{salzer2021reachability}
S{\"a}lzer, M., Lange, M.: Reachability is np-complete even for the simplest
  neural networks. In: International Conference on Reachability Problems. pp.
  149--164. Springer (2021)

\bibitem{schulman2017proximal}
Schulman, J., Wolski, F., Dhariwal, P., Radford, A., Klimov, O.: Proximal
  policy optimization algorithms. arXiv preprint arXiv:1707.06347  (2017)

\bibitem{SoudjaniGA15}
Soudjani, S.E.Z., Gevaerts, C., Abate, A.: {FAUST} \({}^{\mbox{ 2}}\) : Formal
  abstractions of uncountable-state stochastic processes. In: Baier, C.,
  Tinelli, C. (eds.) Tools and Algorithms for the Construction and Analysis of
  Systems - 21st International Conference, {TACAS} 2015, Held as Part of the
  European Joint Conferences on Theory and Practice of Software, {ETAPS} 2015,
  London, UK, April 11-18, 2015. Proceedings. Lecture Notes in Computer
  Science, vol.~9035, pp. 272--286. Springer (2015),
  \url{https://doi.org/10.1007/978-3-662-46681-0\_23}

\bibitem{sutton2018reinforcement}
Sutton, R.S., Barto, A.G.: Reinforcement learning: An introduction. MIT press
  (2018)

\bibitem{SzegedyZSBEGF13}
Szegedy, C., Zaremba, W., Sutskever, I., Bruna, J., Erhan, D., Goodfellow,
  I.J., Fergus, R.: Intriguing properties of neural networks. In: Bengio, Y.,
  LeCun, Y. (eds.) 2nd International Conference on Learning Representations,
  {ICLR} 2014, Banff, AB, Canada, April 14-16, 2014, Conference Track
  Proceedings (2014), \url{http://arxiv.org/abs/1312.6199}

\bibitem{TakisakaOUH21}
Takisaka, T., Oyabu, Y., Urabe, N., Hasuo, I.: Ranking and repulsing
  supermartingales for reachability in randomized programs. {ACM} Trans.
  Program. Lang. Syst.  \textbf{43}(2),  5:1--5:46 (2021),
  \url{https://doi.org/10.1145/3450967}

\bibitem{Turchetta2019SafeEF}
Turchetta, M., Berkenkamp, F., Krause, A.: Safe exploration for interactive
  machine learning. In: NeurIPS (2019)

\bibitem{uchibe2007constrained}
Uchibe, E., Doya, K.: Constrained reinforcement learning from intrinsic and
  extrinsic rewards. In: 2007 IEEE 6th International Conference on Development
  and Learning. pp. 163--168. IEEE (2007)

\bibitem{Vaidya15}
Vaidya, U.: Stochastic stability analysis of discrete-time system using
  lyapunov measure. In: American Control Conference, {ACC} 2015, Chicago, IL,
  USA, July 1-3, 2015. pp. 4646--4651. {IEEE} (2015),
  \url{https://doi.org/10.1109/ACC.2015.7172061}

\bibitem{HBSSH23}
Van~Huijgevoort, B., Sch\"{o}n, O., Soudjani, S., Haesaert, S.: Syscore:
  Synthesis via stochastic coupling relations. In: Proceedings of the 26th ACM
  International Conference on Hybrid Systems: Computation and Control. HSCC
  '23, Association for Computing Machinery (2023),
  \url{https://doi.org/10.1145/3575870.3587123}

\bibitem{VinodGO19}
Vinod, A.P., Gleason, J.D., Oishi, M.M.K.: Sreachtools: a {MATLAB} stochastic
  reachability toolbox. In: Ozay, N., Prabhakar, P. (eds.) Proceedings of the
  22nd {ACM} International Conference on Hybrid Systems: Computation and
  Control, {HSCC} 2019, Montreal, QC, Canada, April 16-18, 2019. pp. 33--38.
  {ACM} (2019), \url{https://doi.org/10.1145/3302504.3311809}

\bibitem{Williams91}
Williams, D.: Probability with Martingales. Cambridge mathematical textbooks,
  Cambridge University Press (1991)

\bibitem{ZikelicLCH22}
Zikelic, D., Lechner, M., Chatterjee, K., Henzinger, T.A.: Learning stabilizing
  policies in stochastic control systems. CoRR  \textbf{abs/2205.11991} (2022),
  \url{https://doi.org/10.48550/arXiv.2205.11991}

\bibitem{ZikelicLHC23}
Zikelic, D., Lechner, M., Henzinger, T.A., Chatterjee, K.: Learning control
  policies for stochastic systems with reach-avoid guarantees. Proceedings of
  the AAAI Conference on Artificial Intelligence  \textbf{37}(10),
  11926--11935 (Jun 2023). \doi{10.1609/aaai.v37i10.26407}

\end{thebibliography}

\appendix

\section{Overview of Probability and Martingale Theory}\label{sec:martingales}

\paragraph{Probability theory} A {\em probability space} is an ordered triple $(\Omega,\mathcal{F},\mathbb{P})$ consisting of a non-empty {\em sample space} $\Omega$, a {\em $\sigma$-algebra} $\mathcal{F}$ over $\Omega$ (i.e.~a collection of subsets of $\Omega$ that contains the empty set $\emptyset$ and is closed under complementation and countable union), and a {\em probability measure} $\mathbb{P}$ over $\mathcal{F}$ which is a function $\mathbb{P}:\mathcal{F}\rightarrow[0,1]$ that satisfies the three Kolmogorov axioms: (1)~$\mathbb{P}[\emptyset]=0$, (2)~$\mathbb{P}[\Omega\backslash A]=1-\mathbb{P}[A]$ for each $A\in\mathcal{F}$, and (3)~$\mathbb{P}[\cup_{i=0}^\infty A_i]=\sum_{i=0}^\infty\mathbb{P}[A_i]$ for any sequence $(A_i)_{i=0}^\infty$ of pairwise disjoint sets in $\mathcal{F}$. Given a probability space $(\Omega,\mathcal{F},\mathbb{P})$, a {\em random variable} is a function $X:\Omega\rightarrow\mathbb{R}\cup\{\pm\infty\}$ that is $\mathcal{F}$-measurable, i.e.~for each $a\in\mathbb{R}$ we have $\{\omega\in\Omega\mid X(\omega)\leq a\}\in\mathcal{F}$. $\mathbb{E}[X]$ denotes the {\em expected value} of $X$. A {\em (discrete-time) stochastic process} is a sequence $(X_i)_{i=0}^{\infty}$ of random variables in $(\Omega,\mathcal{F},\mathbb{P})$.
	
\paragraph{Conditional expectation} Let $(\Omega,\mathcal{F},\mathbb{P})$ be a probability space and $X$ be a random variable in $(\Omega,\mathcal{F},\mathbb{P})$. Given a sub-sigma-algebra $\mathcal{F}'\subseteq\mathcal{F}$, a {\em conditional expectation} of $X$ given $\mathcal{F}'$ is an $\mathcal{F}'$-measurable random variable $Y$ such that, for each $A\in\mathcal{F}'$, we have 
\[ \mathbb{E}[X\cdot\mathbb{I}_A]=\mathbb{E}[Y\cdot\mathbb{I}_A].\]
Here $\mathbb{I}_A:\Omega\rightarrow \{0,1\}$ is an {\em indicator function} of $A$, defined via $\mathbb{I}_A(\omega)=1$ if $\omega\in A$, and $\mathbb{I}_A(\omega)=0$ if $\omega\not\in A$. If $X$ is real-valued and nonnegative, then a conditional expectation of $X$ given $\mathcal{F}'$ exists and is almost-surely unique, i.e.~for any two $\mathcal{F}'$-measurable random variables $Y$ and $Y'$ which are conditional expectations of $X$ given $\mathcal{F}'$ we have that $\mathbb{P}[Y= Y']=1$~\cite{Williams91}. Therefore, we may pick any such random variable as a canonical conditional expectation and denote it by $\mathbb{E}[X\mid \mathcal{F}']$.

\paragraph{Stopping time} A sequence of sigma-algebras $\{\mathcal{F}_i\}_{i=0}^{\infty}$ with $\mathcal{F}_0\subseteq \mathcal{F}_1\subseteq\dots\subseteq\mathcal{F}$ is a {\em filtration} in the probability space $(\Omega,\mathcal{F},\mathbb{P})$. A {\em stopping time} with respect to a filtration $\{\mathcal{F}_i\}_{i=0}^{\infty}$ is a random variable $T:\Omega\rightarrow \mathbb{N}_0\cup\{\infty\}$ such that, for every $i\in\mathbb{N}_0$, we have $\{\omega\in\Omega\mid T(\omega)\leq i\}\in \mathcal{F}_i$. Intuitively, $T$ may be viewed as the time step at which some stochastic process should be ``stopped'', and since $\{\omega\in\Omega\mid T(\omega)\leq i\}\in \mathcal{F}_i$ the decision to stop at the time step $i$ is made solely by using the information available in the first $i$ time steps.

\paragraph{Supermartingales and ranking supermartingales} Let $(\Omega,\mathcal{F},\mathbb{P})$ be a probability space, let $\epsilon\geq 0$ and let $T$ be a stopping time with respect to a filtration $\{\mathcal{F}_i\}_{i=0}^{\infty}$. An {\em $\eps$-ranking supermartingale ($\eps$-RSM) with respect to $T$} is a stochastic process $(X_i)_{i=0}^{\infty}$ such that
\begin{compactitem}
    \item $X_i$ is $\mathcal{F}_i$-measurable, for each $i\geq 0$,
    \item $X_i(\omega)\geq 0$, for each $i\geq 0$ and $\omega\in\Omega$, and
    \item $\mathbb{E}[X_{i+1} \mid \mathcal{F}_i](\omega) \leq X_i(\omega) - \eps\cdot\mathbb{I}_{T>i}(\omega)$, for each $i\geq 0$ and $\omega\in\Omega$.
\end{compactitem}
A {\em supermartingale} with respect to a filtration $\{\mathcal{F}_i\}_{i=0}^{\infty}$ is a stochastic process $(X_i)_{i=0}^{\infty}$ which satisfies conditions~1 and~3 above with $\eps=0$ (thus we define supermartingales only with respect to the filtration and not the stopping time).

We now state two results on RSMs and supermartingales that we will use in our proofs. The first is a result on RSMs that was originally presented in works on termination analysis of probabilistic programs~\cite{FioritiH15,ChatterjeeFNH16}. The second result (see~\cite{Kushner14}, Theorem~7.1) is a concentration bound on the supremum value of a nonnegative supemartingale.

\begin{proposition}\label{prop:rsm}
Let $(\Omega,\mathcal{F},\mathbb{P})$ be a probability space, let $(\mathcal{F}_i)_{i=0}^\infty$ be a filtration and let $T$ be a stopping time with respect to $(\mathcal{F}_i)_{i=0}^\infty$. Suppose that $(X_i)_{i=0}^{\infty}$ is an $\eps$-RSM with respect to $T$, for some $\eps>0$. Then
\begin{compactenum}
    \item $\mathbb{P}[T<\infty] = 1$,
    \item $\mathbb{E}[T] \leq \frac{\mathbb{E}[X_0]}{\eps}$, and
    \item $\mathbb{P}[ T\geq t] \leq \frac{\mathbb{E}[X_0]}{\eps\cdot t}$, for each $t\in\mathbb{N}$.
\end{compactenum}
\end{proposition}

\begin{proposition}\label{prop:bound}
Let $(\Omega,\mathcal{F},\mathbb{P})$ be a probability space and let $(\mathcal{F}_i)_{i=0}^\infty$ be a filtration. Let $(X_i)_{i=0}^{\infty}$ be a nonnegative supermartingale with respect to $(\mathcal{F}_i)_{i=0}^\infty$. Then, for every $\lambda>0$, we have $\mathbb{P}[ \sup_{i\geq 0}X_i \geq \lambda ] \leq \frac{\mathbb{E}[X_0]}{\lambda}$.
\end{proposition}

\section{Proofs of Theorem~\ref{thm:soundness} and Theorem~\ref{thm:bound}}\label{sec:thmproofs}

We now prove Theorem~1\ref{thm:soundness} and Theorem~\ref{thm:bound} from the main text of the paper. For each initial state $\mathbf{x}_0\in\mathcal{X}$, denote by $(\Omega_{\mathbf{x}_0},\mathcal{F}_{\mathbf{x}_0},\mathbb{P}_{\mathbf{x}_0})$ probability space over the set of all system trajectories that start in the initial state $\mathbf{x}_0$ that is induced by the Markov decision process semantics of the system~\cite{Puterman94}. We start both proofs by showing that, for every state $\mathbf{x}_0\in\mathcal{X}\backslash\Stable$, the sRSM $V$ for the set $\Stable$ gives rise to a mathematical RSM in the probability space $(\Omega_{\mathbf{x}_0},\mathcal{F}_{\mathbf{x}_0},\mathbb{P}_{\mathbf{x}_0})$.

\paragraph{Canonical filtration and stopping time} In order to formally show that $V$ can be instantiated as a mathematical RSM in this probability space, we first define the canonical filtration in this probability space and the stopping time with respect to which the mathematical RSM is defined.
Let $\mathbf{x}_0\in\mathcal{X}$ and consider the probability space $(\Omega_{\mathbf{x}_0},\mathcal{F}_{\mathbf{x}_0},\mathbb{P}_{\mathbf{x}_0})$. For each $i\in\mathbb{N}_0$, define $\mathcal{F}_i\subseteq \mathcal{F}$ to be the $\sigma$-algebra containing the subsets of $\Omega_{\mathbf{x}_0}$ that, intuitively, contain all trajectories in $\Omega_{\mathbf{x}_0}$ whose first $i$ states satisfy some specified property. Formally, we define $\mathcal{F}_i$ as follows. For each $j\in\mathbb{N}_0$, let $C_j:\Omega_{\mathbf{x}_0}\rightarrow \mathcal{X}$ be a map which to each trajectory $\rho=(\mathbf{x}_t,\mathbf{u}_t,\omega_t)_{t\in\mathbb{N}_0}\in\Omega_{\mathbf{x}_0}$ assigns the $j$-th state $\mathbf{x}_j$ along the trajectory. Then $\mathcal{F}_i$ is the smallest $\sigma$-algebra over $\Omega_{\mathbf{x}_0}$ with respect to which $C_0, C_1, \dots, C_i$ are all measurable, where $\mathcal{X}\subseteq\mathbb{R}^m$ is equipped with the induced Borel-$\sigma$-algebra (see Section~1,~\cite{Williams91}). Clearly $\mathcal{F}_0\subseteq\mathcal{F}_1\subseteq\dots$. We say that the sequence of $\sigma$-algebras $(\mathcal{F}_i)_{i=0}^\infty$ is the {\em canonical filtration} in the probability space $(\Omega_{\mathbf{x}_0},\mathcal{F}_{\mathbf{x}_0},\mathbb{P}_{\mathbf{x}_0})$.

We then define $T_S:\Omega_{\mathbf{x}_0}\rightarrow\mathbb{N}_0\cup\{\infty\}$ to be the first hitting time of the set $S = \{\mathbf{x}\in\mathcal{X} \mid V(\mathbf{x})\leq M\}$, i.e.~$T_S = \inf \{t\in \mathbb{N}_0 \mid \mathbf{x}_t\in S\}$. Since whether $T_S(\rho)\leq i$ depends solely on the first $i$ states along $\rho$, we clearly have $\{\rho\in\Omega_{\mathbf{x_0}}\mid T_S(\rho)\leq i\}\in \mathcal{F}_i$ for each $i$ and so $T_S$ is a stopping time with respect to $(\mathcal{F}_i)_{i=0}^\infty$.

We now prove the theorems.

\begin{theorem}
If there exist $\epsilon, M, \delta > 0$ and an $(\epsilon,M,\delta)$-sRSM for $\Stable$, then $\Stable$ is a.s.~asymptotically stable.
\end{theorem}

\begin{proof}
We need to show that $\mathbb{P}_{\mathbf{x}_0}[ \lim_{t\rightarrow\infty}d(\mathbf{x}_t,\Stable) = 0 ] = 1$ for every $\mathbf{x}_0\in\mathcal{X}$. We show this by proving the following two claims. First, we show that, from each initial state $\mathbf{x}_0\in\mathcal{X}$, the agent converges to and reaches $S = \{\mathbf{x}\in\mathcal{X} \mid V(\mathbf{x})\leq M\}$ with probability~$1$. The set $S$ is a subset of $\Stable$ by condition~$3$ in Definition~3 of sRSMs. Second, we show that once the agent is in $S$ it may leave $\Stable$ with probability at most $p=\frac{M+L_V\cdot \Delta}{M+L_V\cdot \Delta+\delta}<1$. We then prove that the two claims imply the theorem statement.

\medskip\noindent{\em Claim 1.} For each intial state $\mathbf{x}_0\in\mathcal{X}$, the agent converges to and reaches $S = \{\mathbf{x}\in\mathcal{X} \mid V(\mathbf{x})\leq M\}$ with probability~$1$.
    
\medskip\noindent To prove Claim~1, let $\mathbf{x}_0\in\mathcal{X}$. If $\mathbf{x}_0\in S$, then the claim trivially holds. So suppose w.l.o.g.~that $\mathbf{x}_0\not\in S$. We consider the probability space $(\Omega_{\mathbf{x}_0},\mathcal{F}_{\mathbf{x}_0},\mathbb{P}_{\mathbf{x}_0})$ of all system trajectories that start in $\mathbf{x}_0$, and for each $i\in\mathbb{N}_0$ we define a random variable $X_i$ in this probability space via
\begin{equation}\label{eq:process}
        X_i(\rho) = \begin{cases}
        V(\mathbf{x}_i), &\text{if } i<T_S(\rho) \\
        V(\mathbf{x}_{T_S(\rho)}), &\text{otherwise}
        \end{cases}
\end{equation}
for each trajectory $\rho=(\mathbf{x}_t,\mathbf{u}_t,\omega_t)_{t\in\mathbb{N}_0}\in\Omega_{\mathbf{x}_0}$. In words, $X_i$ is equal to the value of $V$ at the $i$-th state along the trajectory until $S$ is reached, upon which it becomes constant and equal to the value of $V$ upon first entry into $S$. We prove that $(X_i)_{i=0}^\infty$ is an $\eps$-RSM with respect to the stopping time $T_S$. To prove this claim, we check each defining property of $\eps$-RSMs:
\begin{compactitem}
    \item {\em Each $X_i$ is $\mathcal{F}_i$-measurable.} The value of $X_i$ is determined by the first $i$ states along a trajectory, so by the definition of the canonical filtration we have that $X_i$ is $\mathcal{F}_i$-measurable for each $i\geq 0$.
    \item {\em Each $X_i(\rho)\geq 0$.} Since each $X_i$ is defined in terms of $V$ and since we know that $V(\mathbf{x})\geq 0$ for each state $\mathbf{x}\in\mathcal{X}$ by condition~$1$ in Definition~3 of sRSMs, it follows that $X_i(\rho)\geq 0$ for each $i\geq 0$ and $\rho\in\Omega_{\mathbf{x}_0}$.
    \item {\em Each $\mathbb{E}[X_{i+1} \mid \mathcal{F}_i](\rho) \leq X_i(\rho) - \eps\cdot \mathbb{I}_{T_{\Stable}>i}(\rho)$.} First, we remark that the conditional expectation exists since $X_{i+1}$ is nonnegative for each $i\geq 0$. In order to prove the desired inequality, we distinguish between two cases.  Let $\rho=(\mathbf{x}_t,\mathbf{u}_t,\omega_t)_{t\in\mathbb{N}_0}$.
    
    First, consider the case $T_S(\rho)>i$. We have that $X_i(\rho)=V(\mathbf{x}_i)$. On the other hand, we have $\mathbb{E}[X_{i+1} \mid \mathcal{F}_i](\rho) = \mathbb{E}_{\omega\sim d}[V(f(\mathbf{x}_i,\pi(\mathbf{x}_i),\omega)]$. To see this, observe that $\mathbb{E}_{\omega\sim d}[V(f(\mathbf{x}_i,\pi(\mathbf{x}_i),\omega)]$ satisfies all the defining properties of conditional expectation since it is the expected value of $V$ at a subsequent state of $\mathbf{x}_i$, and recall that conditional expectation is a.s.~unique whenever it exists. Hence,
    \begin{equation*}
    \begin{split}
        \mathbb{E}[X_{i+1} \mid \mathcal{F}_i](\rho) &= \mathbb{E}_{\omega\sim d}[V(f(\mathbf{x}_i,\pi(\mathbf{x}_i),\omega)] \\
                                                     &\leq V(\mathbf{x}_i) - \eps = X_i(\rho) - \eps,
    \end{split}
    \end{equation*}
    where the inequality holds by condition~$2$ in Definition~3 of sRSMs and since $\mathbf{x}_i\not\in S$ as $T_S(\rho)>i$. This proves the desired inequality.
    
    Second, consider the case $T_S(\rho)\leq i$. We have $X_i(\rho)=V(\mathbf{x}_{T_S(\rho)})$ and $\mathbb{E}[X_{i+1} \mid \mathcal{F}_i](\rho)] = V(\mathbf{x}_{T_S(\rho)})$, so the desired inequality follows.
\end{compactitem}
Thus, we may use the first part of Proposition~1 to conclude that $\mathbb{P}_{\mathbf{x}_0}[T_S<\infty]=1$, equivalently $\mathbb{P}_{\mathbf{x}_0}[\exists\, t\in\mathbb{N}_0 \text{ s.t. } \mathbf{x}_t\in S] = 1$. This concludes the proof of Claim~1.

\medskip\noindent{\em Claim 2.} $\mathbb{P}_{\mathbf{x}_0}[\exists\, t\in\mathbb{N}_0 \text{ s.t. } \mathbf{x}_t\not\in\Stable] = p < 1$ where $p=\frac{M+L_V\cdot \Delta}{M+L_V\cdot \Delta+\delta}$, for each $\mathbf{x}_0 \in S$.

\medskip\noindent To prove Claim~2, recall that $S = \{\mathbf{x}\in\mathcal{X} \mid V(\mathbf{x})\leq M\}$. Thus, as $V$ is Lipschitz continuous with Lipschitz constant $L_V$ and as $\Delta$ is the maxmial step size of the system, it follows that the value of $V$ upon the agent leaving the set $S$ is $\leq M+L_V\cdot \Delta$. Hence, for the agent to leave $\Stable$ from $\mathbf{x}_0\in S$, it first has to reach a state $\mathbf{x}_1$ with $M<V(\mathbf{x}_1)\leq M+L_V\cdot \Delta$ and then also to reach a state $\mathbf{x}_2\not\in\Stable$ from $\mathbf{x}_1$ without reentering $S$. By condition~$3$ in Definition~3 of sRSMs, we must have $V(\mathbf{x}_2)\geq M+L_V\cdot \Delta + \delta$. Therefore,
\begin{equation*}
\begin{split}
    &\mathbb{P}_{\mathbf{x}_0}\Big[\exists\, t\in\mathbb{N}_0 \text{ s.t. } \mathbf{x}_t\not\in\Stable\Big] \\
    = &\mathbb{P}_{\mathbf{x}_0}\Big[\exists\, t_1, t_2\in\mathbb{N}_0 \text{ s.t. } t_1<t_2 \\
    &\hspace{2cm}\text{ and } M< V(\mathbf{x}_{t_1})\leq M+L_V\cdot \Delta \\
    &\hspace{2cm}\text{ and } V(\mathbf{x}_2)\geq M+L_V\cdot \Delta + \delta \\
    &\hspace{2cm} \text{ with } \mathbf{x}_t\not\in S \text{ for all } t_1 \leq t \leq t_2\Big] \\
    = &\mathbb{P}_{\mathbf{x}_0}\Big[\exists\, t_1\in\mathbb{N}_0 \text{ s.t. } M< V(\mathbf{x}_{t_1})\leq M+L_V\cdot \Delta\Big] \\
    \cdot &\mathbb{P}_{\mathbf{x}_0}\Big[\exists\, t_1, t_2\in\mathbb{N}_0 \text{ s.t. } t_1<t_2 \\
    &\hspace{2cm}\text{ and } M< V(\mathbf{x}_{t_1})\leq M+L_V\cdot \Delta \\
    &\hspace{2cm}\text{ and } V(\mathbf{x}_2)\geq M+L_V\cdot \Delta + \delta \\
    &\hspace{2cm} \text{ with } \mathbf{x}_t\not\in S \text{ for all } t_1 \leq t \leq t_2 \\
    &\hspace{0.3cm}\mid \exists\, t_1\in\mathbb{N}_0 \text{ s.t. } M< V(\mathbf{x}_{t_1})\leq M+L_V\cdot \Delta\Big] \\
    \leq &\mathbb{P}_{\mathbf{x}_0}\Big[\exists\, t_1\in\mathbb{N}_0 \text{ s.t. } M< V(\mathbf{x}_{t_1})\leq M+L_V\cdot \Delta\Big] \\
    \cdot &\sup_{\mathbf{x}_1\in\mathcal{X},\, M< V(\mathbf{x}_{t_1})\leq M+L_V\cdot \Delta} \mathbb{P}_{\mathbf{x}_1}\Big[\exists\, t_2\in\mathbb{N}_0 \text{ s.t. } V(\mathbf{x}_{t_2}) \geq \\
    &M+L_V\cdot \Delta + \delta \text{ and } \mathbf{x}_t\not\in S \text{ for all } 0 \leq t \leq t_2 \Big] \\
    \leq &\sup_{\mathbf{x}_1\in\mathcal{X},\, M< V(\mathbf{x}_{t_1})\leq M+L_V\cdot \Delta} \mathbb{P}_{\mathbf{x}_1}\Big[\exists\, t_2\in\mathbb{N}_0 \text{ s.t. } V(\mathbf{x}_{t_2}) \geq \\
    &M+L_V\cdot \Delta + \delta \text{ and } \mathbf{x}_t\not\in S \text{ for all } 0 \leq t \leq t_2 \Big].
\end{split}
\end{equation*}
The first equality follows by the above observations. The second equality follows by Bayes' rule. The third inequality follows by observing that the trajectory satisfies the Markov property and therefore that the supremum value of $V$ upon visiting a state does not depend on previously visited states. Finally, the fourth inequality follows since the value of the first probability term is $\leq 1$.

Thus, to prove that $\mathbb{P}_{\mathbf{x}_0}[\exists\, t\in\mathbb{N}_0 \text{ s.t. } \mathbf{x}_t\not\in\Stable] = p < 1$ with $p=\frac{M+L_V\cdot \Delta}{M+L_V\cdot \Delta+\delta}$ and therefore conclude Claim~2, it suffices to prove that, for each $\mathbf{x}_1\in\mathcal{X}$ with $M< V(\mathbf{x}_{t_1})\leq M+L_V\cdot \Delta$, we have
\begin{equation*}
\begin{split}
\mathbb{P}_{\mathbf{x}_1}\Big[&\exists\, t_2\in\mathbb{N}_0 \text{ s.t. } V(\mathbf{x}_{t_2})\geq M+L_V\cdot \Delta + \delta  \text{ and } \mathbf{x}_t\not\in S\\
&\text{ for all } 0 \leq t \leq t_2 \Big] \leq \frac{M+L_V\cdot \Delta}{M+L_V\cdot \Delta+\delta}.
\end{split}
\end{equation*}
To prove this, consider now the probability space $(\Omega_{\mathbf{x}_1},\mathcal{F}_{\mathbf{x}_1},\mathbb{P}_{\mathbf{x}_1})$ of all trajectories that start in $\mathbf{x}_1$, the canonical filtration $(\mathcal{F}_i)_{i=0}^\infty$ and the stopping time $T_S$ with respect to it, and define a stochastic process $(X_i)_{i=0}^\infty$ in the probability space via
\begin{equation*}
    X_i(\rho) = \begin{cases}
    V(\mathbf{x}_i), &\text{if } i<T_S(\rho) \\
    V(\mathbf{x}_{T_S(\rho)}), &\text{otherwise}
    \end{cases}
\end{equation*}
for each $i\geq 0$ and a trajectory $\rho$ that starts in $\mathbf{x}_1$. The argument analogous to the proof of Claim~1 shows that it is an $\epsilon$-RSM with respect to the stopping time $T_S$. But note that $\sup_{i\geq 0}X_i$ is equal to the supremum value attained by $V$ until the first hitting time of the set $S$. Hence the above inequality follows immediately from Proposition~2 by observing that $\mathbb{E}_{\mathbf{x}_1}[X_0] = V(\mathbf{x}_1) \leq M+L_V\cdot\Delta$ and plugging in $\lambda = M+L_V\cdot \Delta+\delta$. This concludes the proof of Claim~2.

\medskip\noindent{\em Proof that Claim~1 and Claim~2 imply Theorem~\ref{thm:soundness}.} By Claim~1, the agent with probability $1$ converges to $S\subseteq\Stable$ from any initial state. On the other hand, by Claim~2, upon reaching a state in $S$ the probability of leaving $\Stable$ is at most $p<1$. Finally, by Claim~1 again the agent is guaranteed to converge back to $S$ even upon leaving $\Stable$. Hence, due to the system dynamics under a given policy satisfying Markov property, the probability of the agent leaving and reentering $S$ more than $N$ times is bounded from above by $p^N$. Hence, by letting $N\rightarrow \infty$, we conclude that the probability of the agent leaving $\Stable$ and reentering infinitely many times is $0$, so the agent with probability~$1$ eventually enters and $S$ and does not leave $\Stable$ after that. This implies that $\Stable$ is a.s.~asymptotically stable.

\end{proof}

\begin{theorem}
Let $\epsilon, M, \delta > 0$ and suppose that $V: \mathcal{X} \rightarrow \mathbb{R}$ is an $(\epsilon,M,\delta)$-sRSM for $\Stable$. Let $\Gamma = \sup_{\mathbf{x}\in\Stable}V(\mathbf{x})$ be the supremum of all possible values that $V$ can attain over the stabilizing set $\Stable$. Then, for each initial state $\mathbf{x}_0\in\mathcal{X}$, we have that
\begin{compactenum}
    \item $\mathbb{E}_{\mathbf{x}_0}[\mathsf{Out}_{\Stable}] \leq \frac{V(\mathbf{x}_0)}{\epsilon} + \frac{(M+L_V\cdot \Delta)\cdot (\Gamma + L_V\cdot \Delta)}{\delta\cdot\epsilon}$.
    \item $\mathbb{P}_{\mathbf{x}_0}[\mathsf{Out}_{\Stable} \geq t] \leq \frac{V(\mathbf{x}_0)}{t\cdot\epsilon} + \frac{(M+L_V\cdot \Delta)\cdot (\Gamma + L_V\cdot \Delta)}{\delta\cdot\epsilon\cdot t}$, for any time $t\in\mathbb{N}$.
\end{compactenum}
\end{theorem}

\begin{proof}

We start by proving the first item in Theorem~\ref{thm:bound}. Let $\rho=(\mathbf{x}_t,\mathbf{u}_t,\omega_t)_{t\in\mathbb{N}_0}$ be a system trajectory. Recall that $S = \{\mathbf{x}\in\mathcal{X} \mid V(\mathbf{x})\leq M\}\subseteq\Stable$ and that $T_{S}(\rho)=\inf\{t\in\mathbb{N}_0\mid \mathbf{x}_t\in\Stable\}$
is the first hitting time of $S$. Let us also denote by $\mathsf{OutAfter}_{\Stable}(\rho)=|\{t > T_{S}(\rho) \mid \mathbf{x}_t\not\in\Stable \}|$ the number of time-steps that the trajectory $\rho$ is in states outside of the stabilizing set $\Stable$ after the first hitting time of $S$. Then, since $S\subseteq\Stable$, for each system trajectory $\rho=(\mathbf{x}_t,\mathbf{u}_t,\omega_t)_{t\in\mathbb{N}_0}$ we have that
\[ \mathsf{Out}_{\Stable}(\rho) \leq T_{S}(\rho) + \mathsf{OutAfter}_{\Stable}(\rho). \]
Therefore, for each initial state $\mathbf{x}_0\in\mathcal{X}$, we have
\begin{equation}\label{eq:1}
\begin{split}
    \mathbb{E}_\mathbf{x_0}[\mathsf{Out}_{\Stable}] &\leq \mathbb{E}_\mathbf{x_0}[T_{S}] + \mathbb{E}_\mathbf{x_0}[\mathsf{OutAfter}_{\Stable}] \\
    &\leq \mathbb{E}_\mathbf{x_0}[T_{S}] + \sup_{\mathbf{x}\in\mathcal{X}}\mathbb{E}_\mathbf{x}[\mathsf{OutAfter}_{\Stable}].
\end{split}
\end{equation}

Now, by defining an $\eps$-RSM $(X_i)_{i=0}^\infty$ with respect to the stopping time $T_S$ analogously as in the proof of Theorem~\ref{thm:soundness} and by applying the second item in Proposition~1 to it, we can immediately deduce that
\begin{equation}\label{eq:2}
    \mathbb{E}_\mathbf{x_0}[T_{S}] \leq \frac{\mathbb{E}_\mathbf{x_0}[X_0]}{\epsilon} = \frac{V(\mathbf{x}_0)}{\epsilon}.
\end{equation}

On the other hand, by Claim~2 in the proof of Theorem~\ref{thm:soundness} we know that the probability of leaving $\Stable$ once in $S$ is at most $p = \frac{M+L_V\cdot \Delta}{M+L_V\cdot \Delta+\delta} < 1$. Furthermore, once the stabilizing set $\Stable$ is left, we know that the value of $V$ is at most $\sup_{\mathbf{x}\in\Stable}V(\mathbf{x}) + L_V\cdot\Delta = \Gamma + L_V\cdot\Delta$ due to $L_V$ being the Lipschitz constant of $V$ and $\Delta$ being the maximum step size of the system. Thus, we have
\begin{equation*}
\begin{split}
     &\sup_{\mathbf{x}\in\mathcal{X}}\mathbb{E}_\mathbf{x}[\mathsf{OutAfter}_{\Stable}] \\
     &\leq p \cdot \Big( \sup_{\mathbf{x}\in\mathcal{X} \text{ s.t. } V(\mathbf{x}) \leq \Gamma + L_V\cdot\Delta} \mathbb{E}_\mathbf{x}[T_{S}] + \sup_{\mathbf{x}\in\mathcal{X}}\mathbb{E}_\mathbf{x}[\mathsf{OutAfter}_{\Stable}] \Big) \\
     &\leq p\cdot\Big( \frac{\Gamma + L_V\cdot \Delta}{\epsilon} + \sup_{\mathbf{x}\in\mathcal{X}}\mathbb{E}_\mathbf{x}[\mathsf{OutAfter}_{\Stable}] \Big),
\end{split}
\end{equation*}
where in the second inequality we again use the second item in Proposition~1 but now applied to the $\eps$-RSM $(X_i)_{i=0}^\infty$ with respect to the stopping time $T_S$ defined in the probability space of all system trajectories that start in the initial state $\mathbf{x}$. Hence, by deducting $p \cdot \sup_{\mathbf{x}\in\mathcal{X}}\mathbb{E}_\mathbf{x}[\mathsf{OutAfter}_{\Stable}]$ from both sides of the inequality and then dividing both sides of the resulting inequality by $1-p>0$, we conclude that
\[ \sup_{\mathbf{x}\in\mathcal{X}}\mathbb{E}_\mathbf{x}[\mathsf{OutAfter}_{\Stable}] \leq \frac{p\cdot (\Gamma + L_V\cdot \Delta)}{(1-p)\cdot\epsilon}. \]
Therefore, since $p = \frac{M+L_V\cdot \Delta}{M+L_V\cdot \Delta+\delta}$, we deduce that
\begin{equation}\label{eq:3}
    \sup_{\mathbf{x}\in\mathcal{X}}\mathbb{E}_\mathbf{x}[\mathsf{OutAfter}_{\Stable}] \leq \frac{(M+L_V\cdot \Delta)\cdot (\Gamma + L_V\cdot \Delta)}{\delta\cdot\epsilon}.
\end{equation}
By comgining eq.~\ref{eq:1},~\ref{eq:2} and~\ref{eq:3}, we deduce the first item in Theorem~\ref{thm:bound}.

The second item in Theorem~\ref{thm:bound} follows immediately from the first item in Theorem~\ref{thm:bound} and an application of Markov's inequality which implies that $\mathbb{P}_{\mathbf{x}_0}[\mathsf{Out}_{\Stable} \geq t] \leq \frac{\mathbb{E}_{\mathbf{x}_0}[\mathsf{Out}_{\Stable}]}{t}$ for any $t>0$.

\end{proof}



\section{Regularization Terms}\label{sec:regularization}
Here, we provide details on the two regularization objectives that we add to the training loss.

\paragraph{Global minimum regularization}
We add the term $\loss_{\text{< M}}(\theta, \nu)$
to the loss function, which is an auxiliary loss guiding the learner towards learning an sRSM candidate $V_{\nu}$ that attains the global minimum in the set $\{\mathbf{x}\in\mathcal{X}\mid V(\mathbf{x}) < M\}$. 
In particular, we impose a set $T \subseteq \Stable$ to have value $< M$ and the global minimum of the sRSM being in $T$. 
While this loss term does not enforce any of the conditions in Definition 3 directly, we observe that it helps our learning process. It is defined via
\begin{equation*}
\begin{split}
    \loss_{<M}(\theta, \nu) &= \max\{\max_{x_1,\dots x_{N_3} \in \mathcal{D}_{<M}} V_\nu(x) - M, 0\} +\\
    &\max\{\min_{x_1,\dots x_{N_4} \in \mathcal{X}} V_\nu(x) - \min_{x_1,\dots x_{N_3} \in \mathcal{D}_{<M}} V_\nu(x), 0\}.
\end{split}
\end{equation*}
where $\mathcal{D}_{<M}$ is a set of states at which the sRSM canidate learned in the previous learning iteration is $<M$ and $N_3$ and $N_4$ are algorithm parameters.

\paragraph{Lipschitz regularization}
We regularize Lipschitz bounds of $V_\nu$ and $\pi_\theta$ during trainin by adding the regularization term
\begin{equation}
    \lambda (\loss_{\text{Lipschitz}}(\theta) + \loss_{\text{Lipschitz}}(\nu) ) + \alpha \loss'_{\text{Lipschitz}}(\nu),
\end{equation}
to the training objective, with 
\begin{equation*}
    \loss_{\text{Lipschitz}}(\phi) = \max\Big\{  \prod_{W,b \in \phi} \max_j \sum_{i} |W_{i,j}| - \rho, 0 \Big\}
\end{equation*}
and 
\begin{equation*}
    \loss'_{\text{Lipschitz}}(\phi) = \min\Big\{  \prod_{W,b \in \phi} \max_j \sum_{i} |W_{i,j}| - \rho', 0 \Big\}.
\end{equation*}

\section{Proof of Theorem~3}\label{app:thmthreeproof}

\begin{theorem}[Algorithm correctness]
Suppose that the verifier shows that $V_{\nu}$ satisfies \eqref{eq:expdecstricter} for each $\tilde{\mathbf{x}} \in \tilde{\mathcal{X}}_{\geq M}$ and \eqref{eq:stablecond} for each $\text{cell}\in\text{Cells}_{\mathcal{X}\backslash\Stable}$, so Algorithm~\ref{alg:algorithm} returns $\pi_\theta$ and $V_{\nu}$. Then $V_{\nu}$ is an sRSM and $\Stable$ is a.s.~asymptotically stable under~$\pi_{\theta}$.
\end{theorem}

\begin{proof}
To prove the theorem, we first need to show that $V_{\nu}$ satisfies the three conditions in Definition~3.

Condition~1 in Definition~3 is satisfied by default since $V_{\nu}$ applies the softplus activation function to its output which ensures nonnegativity.

To deduce condition~2 in Definition~3, we need to show that there exists $\epsilon>0$ such that for each $\mathbf{x}\in\mathcal{X}$ with $V_{\nu}(\mathbf{x}) \geq M$ we have
    \begin{equation*}
    \mathbb{E}_{\omega\sim d}\Big[ V_{\nu} \Big( f(\mathbf{x}, \pi(\mathbf{x}), \omega) \Big) \Big] \leq V(\mathbf{x}) - \epsilon.
    \end{equation*}
We show that
\[ \epsilon = \min_{\tilde{\mathbf{x}}\in \tilde{\mathcal{X}}_{\geq M}} \Big( V(\tilde{\mathbf{x}}) - \tau \cdot K - \mathbb{E}_{\omega\sim d}\Big[ V \Big( f(\tilde{\mathbf{x}}, \pi(\tilde{\mathbf{x}}), \omega) \Big) \Big] \Big) \]
satisfies this requirement. Fix $\mathbf{x}\in\mathcal{X}$ with $V_{\nu}(\mathbf{x}) \geq M$ and let $\tilde{\mathbf{x}}\in\tilde{\mathcal{X}}$ be such that $||\mathbf{x}-\tilde{\mathbf{x}}||_1 \leq \tau$. Such $\tilde{\mathbf{x}}$ exists by definition of a discretization. Furthremore, since $V_{\nu}(\mathbf{x}) \geq M$, the center of the cell that contains $\mathbf{x}$ must be contained in $\tilde{\mathcal{X}}_{\geq M}$ so therefore we may pick such $\tilde{\mathbf{x}}\in\tilde{\mathcal{X}}_{\geq M}$ (the correctness of the computation of $\tilde{\mathcal{X}}_{\geq M}$ follows from the correctness of IA-AI~\cite{CousotC77,Gowal18}). Then, by Lipschitz continuity of $f$, $\pi_{\theta}$ and $V_{\nu}$, we have that
\begin{equation}\label{eq:long}
\begin{split}
    &\mathbb{E}_{\omega\sim d}\Big[ V_{\nu} \Big( f(\mathbf{x}, \pi_{\theta}(\mathbf{x}), \omega) \Big) \Big] \\
    &\leq \mathbb{E}_{\omega\sim d}\Big[ V_{\nu} \Big( f(\tilde{\mathbf{x}}, \pi_{\theta}(\tilde{\mathbf{x}}), \omega) \Big) \Big] \\
    &+ ||f(\tilde{\mathbf{x}}, \pi_{\theta}(\tilde{\mathbf{x}}), \omega) - f(\mathbf{x}, \pi(\mathbf{x}), \omega)||_1 \cdot L_V \\
    &\leq \mathbb{E}_{\omega\sim d}\Big[ V_{\nu} \Big( f(\tilde{\mathbf{x}}, \pi_{\theta}(\tilde{\mathbf{x}}), \omega) \Big) \Big] \\
    &+ ||(\tilde{\mathbf{x}}, \pi_{\theta}(\tilde{\mathbf{x}}), \omega) - (\mathbf{x}, \pi(\mathbf{x}), \omega)||_1 \cdot L_V\cdot L_f \\
    &\leq \mathbb{E}_{\omega\sim d}\Big[ V_{\nu} \Big( f(\tilde{\mathbf{x}}, \pi_{\theta}(\tilde{\mathbf{x}}), \omega) \Big) \Big] \\
    &+ ||\tilde{\mathbf{x}} - \mathbf{x}||_1 \cdot L_V\cdot L_f \cdot (1 + L_{\pi}) \\
    &\leq \mathbb{E}_{\omega\sim d}\Big[ V_{\nu} \Big( f(\tilde{\mathbf{x}}, \pi_{\theta}(\tilde{\mathbf{x}}), \omega) \Big) \Big] \\
    &+ \tau \cdot L_V\cdot L_f \cdot (1 + L_{\pi}),
\end{split}
\end{equation}
On the other hand, by Lipschitz continuity of $V_{\nu}$ we have
\begin{equation}\label{eq:short}
    V_{\nu}(\mathbf{x}) \geq V_{\nu}(\tilde{\mathbf{x}}) - ||\tilde{\mathbf{x}} - \mathbf{x}||_1 \cdot L_V \geq V_{\nu}(\tilde{\mathbf{x}}) - \tau\cdot L_V.
\end{equation}
Thus combining eq.(\ref{eq:long}) and (\ref{eq:short}) we get that
\begin{equation}
\begin{split}
    &V_{\nu}(\mathbf{x}) - \mathbb{E}_{\omega\sim d}\Big[ V_{\nu} \Big( f(\mathbf{x}, \pi_{\theta}(\mathbf{x}), \omega) \Big) \Big] \\
    &\geq V_{\nu}(\tilde{\mathbf{x}}) - \tau\cdot L_V - \mathbb{E}_{\omega\sim d}\Big[ V_{\nu} \Big( f(\tilde{\mathbf{x}}, \pi_{\theta}(\tilde{\mathbf{x}}), \omega) \Big) \Big] \\
    &- \tau \cdot L_V\cdot L_f \cdot (1 + L_{\pi}) \\
    &= V_{\nu}(\tilde{\mathbf{x}}) - \tau\cdot K - \mathbb{E}_{\omega\sim d}\Big[ V_{\nu} \Big( f(\tilde{\mathbf{x}}, \pi_{\theta}(\tilde{\mathbf{x}}), \omega) \Big) \Big] \\
    &\geq \epsilon,
\end{split}
\end{equation}
The last inequality holds by our definition of $\epsilon$, therefore we conclude that $V_{\nu}$ satisfies condition~2 in Definition~3.

Finally, to deduce condition~3 in Definition~3, we need to show that there exists $\delta>0$ such that $V_{\nu}(\mathbf{x})\geq M + L_V\cdot\Delta + \delta$ holds for each $\mathbf{x}\in\mathcal{X}\backslash\Stable$. But the fact that
\[ \delta=\min_{\text{cell}\in\text{Cells}_{\mathcal{X}\backslash\Stable}}\{\underline{V}_{\nu}(\text{cell}) - M - L_V\cdot \Delta_{\theta}\} \]
satisfies the claim follows immediately from correctness of IA-AI and the fact that eq.~(3) holds for each $\text{cell}\in\text{Cells}_{\mathcal{X}\backslash\Stable}$. 

Thus, this concludes the proof that $V_{\nu}$ satisfies the three conditions in Definition~3. Then, by Theorem~\ref{thm:soundness} on sRSMs, we know that $\Stable$ is a.s.~asymptotically stable under $\pi_{\theta}$.
\end{proof}

\section{Experimental evaluation details}\label{app:experimentsdetails}

We implemented our algorithm in JAX. All experiments were run on a 4 CPU-core machine with 64GB of memory and an NVIDIA A10 with 24GB of memory.

\paragraph{Benchmark environments}
The dynamics of the two-dimensional dynamical system (2D system) are defined as

\begin{equation}
\begin{split}
    \mathbf{x}_{t+1} &= \begin{pmatrix}
        1 & 0.0196 \\
        0 & 0.98
    \end{pmatrix}\mathbf{x}_t + \begin{pmatrix}
        0.002 \\
        0.1 
    \end{pmatrix}g(\mathbf{u}_t) \\
    &+ \begin{pmatrix}
        0.002 & 0 \\
        0 & 0.001 
    \end{pmatrix}\omega,
\end{split}
\end{equation}
where $\omega$ is a disturbance vector and $\omega[1], \omega[2] \sim \text{Triangular}$. 
The function $g$ bounds the range of admissible actions by $g(u) = \max(\min(u, 1), -1)$.

The probability density function of $\text{Triangular}$ is defined by 
\begin{equation}
    \text{Triangular}(x) := \begin{cases} 0 & \text{if } x< -1\\ 1 - |x| & \text{if } -1 \leq x \leq 1\\ 0 & \text{otherwise}\end{cases}.
\end{equation}

The dynamics function of the inverted pendulum task is defined as 
\begin{align*}
    \mathbf{x}_{t+1}[2] &:= (1-b)  \mathbf{x}_{t}[2] \\
    &+ d \cdot \big( \frac{-1.5 \cdot G \cdot \text{sin}(\mathbf{x}_{t}[1]+\pi)}{2l} + \frac{3}{m l^2} 2g(\mathbf{u}_t)\big)\\
    &+ 0.002 \omega[1]\\
    \mathbf{x}_{t+1}[1] &:=  \mathbf{x}_{t}[1] + d \cdot \mathbf{x}_{t+1}[2] + 0.005 \omega[2],
\end{align*}
where the parameters $d, G, m, l, b$ are defined in Table \ref{tab:invpend}.
For training a policy on the inverted pendulum task, we used a reward $r_t$ at time $t$ defined by $r_t := 1 - \mathbf{x}_{t}[1]^2 - 0.1 \mathbf{x}_{t}[2]^2$.

\begin{table}
    \centering
    \begin{tabular}{c|c}\toprule
        Parameter & Value  \\\midrule
        $d$ &  0.05\\
        $G$ & 10\\
        $m$ & 0.15\\
        $l$ & 0.5\\
        $b$ & 0.1\\\bottomrule
    \end{tabular}
    \caption{Parameters of the inverted pendulum task.}
    \label{tab:invpend}
\end{table}

The hyperparameters we used in the experiments for learning the policy and the sRSM are listed in Table \ref{tab:hp}.
For each of the tasks, we consider $T = \{x \mid |x_1| \le 0.2, |x_2| \le 0.2\}$.

\begin{table}[!ht]
  \begin{center}
    \begin{tabular}{c|c}
      \toprule
      Parameter & Value\\
      \midrule
      Learning rate & 0.0005 \\
      $\lambda$ & 0.001 \\
      $\alpha$ & 10 \\
      $\rho_\theta$ & 4 \\
      $\rho_\nu$ & 8 \\
      $\rho'$ & 0.01 \\
      $\delta_\text{train}$ & 0.1 \\
      $N_\text{cond 2}$ & 16 \\
      $N_\text{cond 3}$ & 256 \\
      $N_3$ & 256 \\
      $N_4$ & 512 \\
      $\epsilon_\text{train}$ & 0.1 \\
      \bottomrule
    \end{tabular}
    \caption{Hyperparameters used in our experiments.}
        \label{tab:hp}
  \end{center}
\end{table}
We observed a better convergence and more stable training when training only the sRSM candidate and keep the weights of the policy frozen for the first three iterations of our algorithm. For the second task we replaced $\epsilon_\text{train}$ with $K_{\theta, \nu} \cdot \tau$ during the training. Specifically, instead of using $L_\text{cond 2}(\theta, \nu)$, we set

\begin{equation*}
\begin{split}
    &\loss'_{\text{cond~2}}(\theta, \nu) = \frac{1}{|B|}\sum_{\mathbf{x}\in B}\Big( \max\Big\{ \sum_{\omega_1,\dots, \omega_{N_{\text{cond~2}}} \sim d}\\
    &\frac{V_{\nu}\big(f(\mathbf{x},\pi_\theta(\mathbf{x}),\omega_i)\big)}{N_{\text{cond~2}}} -  V_{\nu}(\mathbf{x})  + K_{\theta, \nu} \cdot \tau, 0\Big\} \Big).
\end{split}
\end{equation*}

For the inverted pendulum task, the plots and the results in Table 1 in the main paper are obtained by training with $\loss'_{\text{cond~2}}(\theta, \nu)$ as the loss function. Here, we performed an ablation study to test whether using $\loss'_{\text{cond~2}}(\theta, \nu)$ can improve the results, i.e., whether the number of iterations is decreased. The results in Table 3 show that the effectiveness of using $\loss'_{\text{cond~2}}(\theta, \nu)$ on the particular system.

\begin{table*}
    \centering
    \begin{tabular}{c|c c c c c}
      \toprule
      \multirow{1}{6em}{Environment} & Use $\loss'_{\text{cond~2}}(\theta, \nu)$ & Iterations & Mesh ($\tau$) & $p$ & Runtime\\
      \midrule
      \multirow{2}{5em}{2D system} & No & 5 & 0.0007 & 0.80 & 3660 s\\
      & Yes & 7 & 0.0007 & 0.78 & 4405 s\\
      \hline
      \multirow{2}{8em}{Inverted pendulum} & No &  8 & 0.003 & 0.97 & 7004 s \\
      & Yes & 4  & 0.003 & 0.97 & 2619 s\\
      \bottomrule
    \end{tabular}
    \caption{Ablation analysis of the impact of the loss term $\loss'_{\text{cond~2}}(\theta, \nu)$. Number of learner-verifier loop iterations, mesh of the discretization used by the verifier, $p$, and total algorithm runtime (in seconds).}
        \label{tab:table11}
\end{table*}

\paragraph{Grid refinement}

We implemented two types of grid refinement procedures to refine the mesh of the discretization used by the verifier. The first refinement is scheduled to multiply $\tau$ by 0.5 every second iteration starting at iteration 5 if no \emph{hard violation} is encountered by the verifier module.
A violation is a counterexample to condition 2 in Definition 3 in the main paper. Hard violations are violations that also violate the condition

\begin{equation*}
    \mathbb{E}_{\omega\sim d}\Big[ V\Big( f(\mathbf{x}, \pi(\mathbf{x}), \omega) \Big) \Big] < V(\mathbf{x}).
    \end{equation*}

Our second refinement procedure is invoked when there are violations but no hard violations. In this case, our procedure tries to verify grid cells where violations were observed using a mesh of $0.5\tau$.

\subsection{PPO Details}\label{app:ppo}
The settings used for the PPO \cite{schulman2017proximal} pre-training process are as follows.
In each PPO iteration, 30 episodes of the environment are collected in a training buffer. Stochastic is introduced to the sampling of the policy network $\pi_\mu$ using a Gaussian distributed random variable added to the policy's output, i.e., the policy predicts a Gaussian's mean. 
The standard deviation of the Gaussian is dynamic during the policy training process according to a linear decay starting from 0.5 at first PPO iteration to 0.05 at PPO iteration 50. The advantage values are normalized by subtracting the mean and scaling by the inverse of the standard deviation of the advantage values of the training buffer. The PPO clipping value $\varepsilon$ is 0.2 and $\gamma$ is set to 0.99.
In each PPO iteration, we train the policy for 10 epochs, except for the first iteration where we train the policy for 30 epochs. An epoch accounts to a pass over the entire data in the training buffer, i.e., the data from the the rollout episodes. 
We train the value network 5 epochs, expect in the first PPO iteration, where we train the value network for 10 epochs. The Lipschitz regularization is applied to the learning of the policy parameters during the PPO pre-training.

\section{Additional plots}
In this section, we include an additional plot visualizing the sRSM learned for the 2D system in Figure \ref{fig:lds1}.
\begin{figure}
    \centering
    \includegraphics[width=\linewidth]{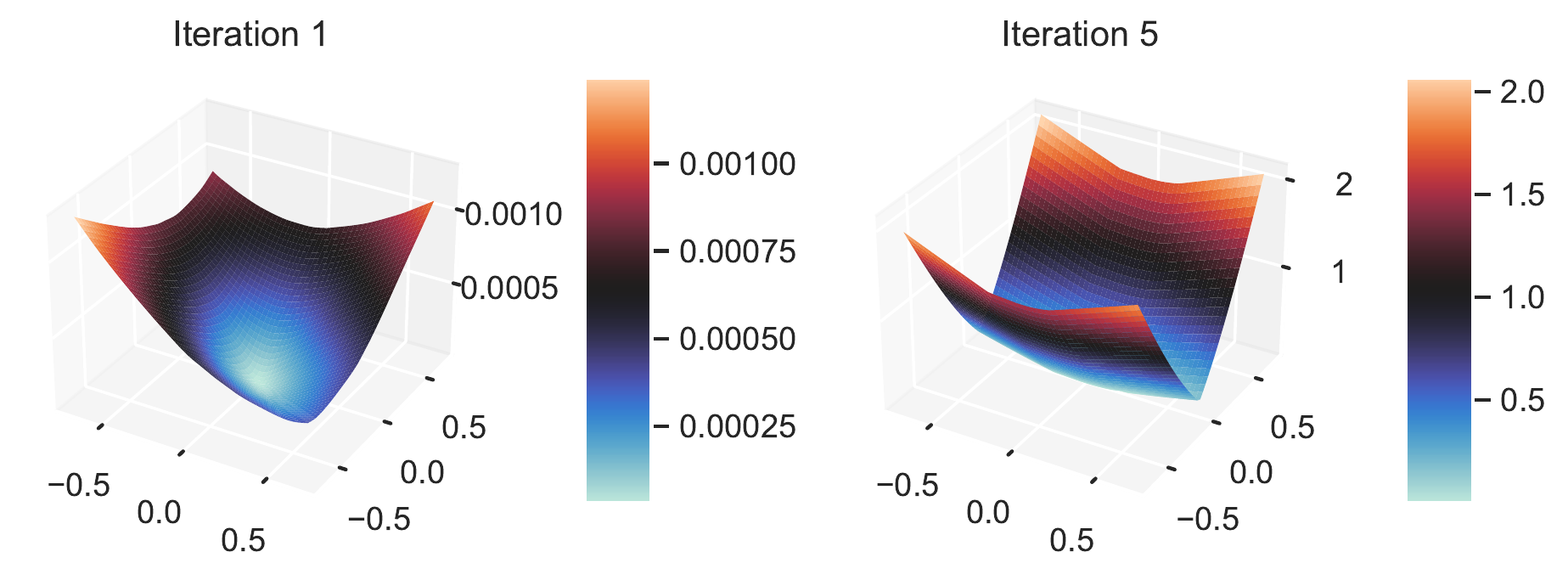}
   \caption{Visualization of the sRSM candidate after 1 and 5 iterations of our algorithm for the 2D system task. The candidate after 1 iteration does not fulfill all sRSM conditions, while the function after 5 learning iterations is a valid sRSM.}
    \label{fig:lds1}
\end{figure}

\end{document}